-\pdfoutput=1
\documentclass{article}

\usepackage{microtype}
\usepackage{graphicx}
\usepackage{subfigure}
\usepackage{booktabs} 

\usepackage{hyperref}

\usepackage[accepted]{icml2019}

\usepackage{times}
\usepackage{helvet}
\usepackage{courier}

\usepackage{algorithm}
\usepackage{algorithmic}

\usepackage[utf8]{inputenc} 
\usepackage[T1]{fontenc}    

\usepackage{booktabs}       
\usepackage{amsfonts}       
\usepackage{nicefrac}       
\usepackage{microtype} 

\usepackage{amsmath}
\usepackage{amssymb}
\usepackage{amsthm}
\usepackage{thmtools,thm-restate}

\usepackage{makeidx}
\usepackage{graphicx}

\usepackage{calc}
\usepackage{multirow}

\usepackage[font=footnotesize, skip=5pt]{caption}

\usepackage{multirow}

\usepackage{color}
\usepackage{xcolor}
\usepackage{colortbl}
\definecolor{lb}{RGB}{155,206,227}
\definecolor{db}{RGB}{31,120,180}
\definecolor{lg}{RGB}{178,223,138}
\definecolor{dg}{RGB}{51,160,44}
\definecolor{darkblue}{rgb}{0,0,0.75}

\definecolor{caddendum}{RGB}{0,0,0}
\definecolor{cerratum}{RGB}{0,0,0}
\definecolor{c-lss-1}{RGB}{28,144,153}
\definecolor{c-lss}{RGB}{255,0,0}
\definecolor{c-dpp}{RGB}{127,121,73}
\definecolor{c-dpp-1}{RGB}{127,121,73}
\definecolor{c-km}{RGB}{117,107,177}
\definecolor{lloyd-km}{RGB}{241,163,64}

\usepackage[american, english]{babel}

\newtheorem{definition}{Definition}
\newtheorem{theorem}{Theorem}
\newtheorem*{theorem*}{Theorem}

\DeclareMathOperator*{\argmin}{arg\,min}

\newcommand{\brackets}[1]{\left( {#1} \right)}
\newcommand{\Brackets}[1]{\Big( {#1} \Big)}
\newcommand{\cbrackets}[1]{\left\{ {#1} \right\} }

\newcommand{\norm}[1]{\left\Vert {#1} \right\Vert}
\newcommand{\absolute}[1]{\left\vert {#1} \right\vert}
\newcommand{\ip}[2]{\left\langle {#1}, {#2} \right\rangle}

\newcommand{\ceq}[1]{(\ref{#1})}
\newcommand{\acro}[1]{\textsc{\MakeLowercase{#1}}}

\usepackage{tikz}
\usepackage{pgfplots}
\usepackage{pgf}
\usepackage{adjustbox}
\usepackage{pgfplotstable}
\pgfplotsset{compat=newest}
\pgfplotsset{
	tick label style={font=\small},
	label style={font=\small},
	legend style={font=\small},
	every axis/.append style={
		thick,
		tick style={semithick, black},
		axis line style={-},
		axis x line =bottom,
		axis y line =left
	}
}

\usepgfplotslibrary{groupplots}
\usepgfplotslibrary{fillbetween}

\newcommand{\krein}{Kre\u{\i}n }

\newlength\figureheight
\newlength\figurewidth

\icmltitlerunning{Scalable Learning in Reproducing Kernel Krein Spaces}

\begin{document}

\twocolumn[
\icmltitle{Scalable Learning in Reproducing Kernel \krein Spaces}

\begin{icmlauthorlist}
\icmlauthor{Dino Oglic}{kcl}
\icmlauthor{Thomas G\"artner}{nott}
\end{icmlauthorlist}

\icmlaffiliation{kcl}{Department of Informatics, King's College London, UK}
\icmlaffiliation{nott}{School of Computer Science, University of Nottingham, UK}

\icmlcorrespondingauthor{Dino Oglic}{dino.oglic@uni-bonn.de}

\vskip 0.3in
]

\printAffiliationsAndNotice{}

\begin{abstract}
We provide the first mathematically complete derivation of the Nystr\"om method for low-rank approximation of indefinite kernels and propose an efficient method for finding an approximate eigendecomposition of such kernel matrices. Building on this result, we devise highly scalable methods for learning in reproducing kernel \krein spaces. The devised approaches provide a principled and theoretically well-founded means to tackle large scale learning problems with indefinite kernels. The main motivation for our work comes from problems with structured representations (e.g., graphs, strings, time-series), where it is relatively easy to devise a pairwise (dis)similarity function based on intuition and/or knowledge of domain experts. Such functions are typically not positive definite and it is often well beyond the expertise of practitioners to verify this condition. The effectiveness of the devised approaches is evaluated empirically using indefinite kernels defined on structured and vectorial data representations.
\end{abstract}

\section{Introduction}
\label{sec:intro}

In learning problems with structured data it is relatively easy to devise a pairwise similarity/dissimilarity function based on intuition/knowledge of domain experts. Such functions are typically not positive definite and it is often the case that verifying this condition is well beyond the expertise of practitioners. The learning problems with indefinite similarity/dissimilarity functions are typically modeled via \krein spaces~\citep[e.g., see][]{Ong04,Loosli16,OglicG18icml}, which are vector spaces with an indefinite bilinear form~\citep{Azizov1981,Iokhvidov82}. The computational and space complexities of these approaches are similar to those of the standard kernel methods that work with positive definite kernels~\citep{Scholkopf01}. The Nystr\"om method~\citep{Nystrom30,Smola00,Williams01} is an effective approach for low-rank approximation of positive definite kernels that can scale kernel methods to problems with millions of instances~\citep{Scholkopf01}. We provide the first mathematically complete derivation that extends the Nystr\"om method to low-rank approximation of indefinite kernels and propose an efficient method for finding an approximate eigendecomposition of such kernel matrices. To tackle the computational issues arising in large scale problems with indefinite kernels, we also devise several novel Nystr\"om-based low-rank approaches tailored for scalable learning in reproducing kernel \krein spaces.

We start by showing that the Nystr\"om method can be used for low-rank approximations of indefinite kernel matrices and provide means for finding their approximate eigendecompositions (Section~\ref{subsec:nystroem}). We then devise two landmark sampling strategies based on state-of-the-art techniques~\citep{Gittens16,OglicG17icml} used in Nystr\"om approximations of positive definite kernels (Section~\ref{subsec:landmarks}). Having described means for finding low-rank factorizations of indefinite kernel matrices, we formulate low-rank variants of two least squares methods~\citep{tikhonov,Rifkin02,OglicG18icml} for learning in reproducing kernel \krein spaces (Section~\ref{subsec:krein-regression}). We also derive a novel low-rank variant of the support vector machine for scalable learning in reproducing kernel \krein spaces (Section~\ref{subsec:krein-classification}), inspired by~\citet{OglicG18icml}. Having introduced means for scalable learning in reproducing kernel \krein spaces, we evaluate the effectiveness of these approaches and the Nystr\"om low-rank approximations on datasets from standard machine learning repositories (Section~\ref{sec:experiments}). The empirical results demonstrate the effectiveness of the proposed approaches in: \emph{i}) classification tasks and \emph{ii}) problems of  finding a low-rank approximation of an indefinite kernel matrix. The experiments are performed using $15$ representative datasets and a variety of indefinite kernels. The paper concludes with a discussion where we contrast ours and other relevant approaches (Section~\ref{sec:discussion}).

\section{Scalable Learning in \krein Spaces}
\label{sec:large-scale}

In this section, we first provide a novel derivation of the Nystr\"om method that allows us to extend the approach to low-rank approximation of indefinite kernel matrices. Building on this result, we then provide means to scale \krein kernel methods to datasets with millions of instances/pairwise (dis)similarities. More specifically, we devise low-rank variants of kernel ridge regression and support vector machines in reproducing kernel \krein spaces, as well as a low-rank variant of the variance constrained ridge regression proposed in~\citet{OglicG18icml}. As the effectiveness of low-rank approximations based on the Nystr\"om method critically depends on the selected landmarks, we also adapt two state-of-the-art landmark sampling strategies proposed for the approximation of positive definite kernels. In addition to this, we make a theoretical contribution (Proposition~\ref{prop:krein-svm}) relating the \krein support vector machines~\citep{Loosli16} to previous work on learning with the flip-spectrum kernel matrices~\citep{GraepelHBO98,Chen09}.

\subsection{Reproducing Kernel \krein Spaces}
\label{subsec:krein-kernels}

Let $\mathcal{K}$ be a vector space defined on the scalar field $\mathbb{R}$. A bilinear form on $\mathcal{K}$ is a function $\ip{\cdot}{\cdot}_{\mathcal{K}} \colon \mathcal{K} \times \mathcal{K} \rightarrow \mathbb{R}$ such that, for all $f, g, h \in \mathcal{K}$ and scalars $\alpha, \beta \in \mathbb{R}$, it holds: \emph{i})
$ \ip{\alpha f + \beta g}{h}_{\mathcal{K}} =\alpha \ip{f}{h}_{\mathcal{K}}  + \beta \ip{g}{h}_{\mathcal{K}}$ and \emph{ii}) $\ip{f}{\alpha g + \beta h}_{\mathcal{K}} =\alpha \ip{f}{g}_{\mathcal{K}} + \beta \ip{f}{h}_{\mathcal{K}} $. The bilinear form is called non-degenerate if
\begin{align*}
	\brackets{\forall f \in  \mathcal{K}} \colon \Brackets{ \brackets{\forall g \in \mathcal{K}}\  \ip{f}{g}_{\mathcal{K}}=0 } \  \Longrightarrow \ f=0.
\end{align*}
The bilinear form $\ip{\cdot}{\cdot}_{\mathcal{K}}$ is symmetric if, for all $f, g \in \mathcal{K}$, we have $ \ip{f}{g}_{\mathcal{K}}=\ip{g}{f}_{\mathcal{K}}$. The form is called indefinite if there exists $f, g \in \mathcal{K}$ such that $\ip{f}{f}_{\mathcal{K}} > 0$ and $\ip{g}{g}_{\mathcal{K}} < 0$. On the other hand, if $\ip{f}{f}_{\mathcal{K}}\geq 0$ for all $f \in \mathcal{K}$, then the form is called positive. A non-degenerate, symmetric, and positive bilinear form on $\mathcal{K}$ is called inner product. Any two elements $f, g \in \mathcal{K}$ that satisfy $\ip{f}{g}_{\mathcal{K}}=0$ are $\ip{\cdot}{\cdot}_{\mathcal{K}}$-orthogonal. Similarly, any two subspaces $\mathcal{K}_1, \mathcal{K}_2 \subset \mathcal{K}$ that satisfy $\ip{f_1}{f_2}_{\mathcal{K}}=0$ for all $f_1\in \mathcal{K}_1$ and $f_2 \in \mathcal{K}_2$ are called $\ip{\cdot}{\cdot}_{\mathcal{K}}$-orthogonal. Having reviewed bilinear forms, we are now ready to introduce the notion of a \krein space.

\begin{definition}~\citep{bognar74,Azizov1981} The vector space $\mathcal{K}$ with a bilinear form $\ip{\cdot}{\cdot}_{\mathcal{K}}$ is called \krein space if it admits a decomposition into a direct sum $\mathcal{K}=\mathcal{H}_{+} \oplus \mathcal{H}_{-}$ of $\ip{\cdot}{\cdot}_{\mathcal{K}}$-orthogonal Hilbert spaces $\mathcal{H}_{\pm}$ such that the bilinear form can be written as 
\begin{align*}
\ip{f}{g}_{\mathcal{K}}=\ip{f_{+}}{g_{+}}_{\mathcal{H}_{+}} - \ip{f_{-}}{g_{-}}_{\mathcal{H}_{-}} \ ,
\end{align*}
where $\mathcal{H}_{\pm}$ are endowed with inner products $\ip{\cdot}{\cdot}_{\mathcal{H}_{\pm}}$, $f=f_{+} \oplus f_{-}$, $g=g_{+} \oplus g_{-}$, and $f_{\pm}, g_{\pm} \in \mathcal{H}_{\pm}$. 
\end{definition}

For a fixed decomposition $\mathcal{K}=\mathcal{H}_{+} \oplus \mathcal{H}_{-}$, the Hilbert space $\mathcal{H}_{\mathcal{K}}=\mathcal{H}_{+} \oplus \mathcal{H}_{-}$ endowed with inner product  
\begin{align*}
\ip{f}{g}_{\mathcal{H}_{\mathcal{K}}} = \ip{f_{+}}{g_{+}}_{\mathcal{H}_+} + \ip{f_{-}}{g_{-}}_{\mathcal{H}_-} \quad (f_{\pm}, g_{\pm} \in \mathcal{H}_{\pm})
\end{align*}
can be associated with $\mathcal{K}$. For a \krein space $\mathcal{K}$, the decomposition $\mathcal{K}=\mathcal{H}_{+} \oplus \mathcal{H}_{-}$ is not necessarily unique. Thus, a \krein space can, in general, be associated with infinitely many Hilbert spaces. However, for any such Hilbert space $\mathcal{H}_{\mathcal{K}}$ the topology introduced on $\mathcal{K}$ via the norm $\norm{f}_{\mathcal{H}_{\mathcal{K}}}=\sqrt{\ip{f}{f}_{\mathcal{H}_{\mathcal{K}}}}$ is independent of the decomposition and the associated Hilbert space. More specifically, all norms $\norm{\cdot}_{\mathcal{H}_{\mathcal{K}}}$ generated by different decompositions of $\mathcal{K}$ into direct sums of Hilbert spaces are topologically equivalent~\citep{Langer1962}. The topology on $\mathcal{K}$ defined by the norm of an associated Hilbert space is called the strong topology. Having reviewed basic properties of \krein spaces, we are now ready to introduce a notion of reproducing kernel \krein space. For that, let $\mathcal{X}$ be an instance space and denote with $\mathbb{R}^{\mathcal{X}}$ the set of functions from $\mathcal{X}$ to $\mathbb{R}$. For a fixed element $x \in \mathcal{X}$, the map $\mathbb{T}_x \colon \mathbb{R}^{\mathcal{X}} \rightarrow \mathbb{R}$ that assigns a real number to each function $f \in \mathcal{R}^{\mathcal{X}}$ is called the evaluation functional at $x$, i.e., $\mathbb{T}_x\brackets{f}=f\brackets{x}$ for all $f \in \mathbb{R}^{\mathcal{X}}$. 

\begin{definition}~\citep{alpay1991,Ong04} A \krein space $\brackets{\mathcal{K}, \ip{\cdot}{\cdot}_{\mathcal{K}}}$ is a reproducing kernel \krein space if $\mathcal{K}\subset \mathbb{R}^{\mathcal{X}}$ and the evaluation functional is continuous on $\mathcal{K}$ with respect to the strong topology.
\end{definition}

The following theorem provides a characterization of reproducing kernel \krein spaces.

\begin{theorem}~\citep{Schwartz1964,alpay1991} Let $k \colon \mathcal{X} \times \mathcal{X} \rightarrow \mathbb{R}$ be a real-valued symmetric function. Then, there is an associated reproducing kernel \krein space if and only if $k = k_{+} - k_{-}$, where $k_{+}$ and $k_{-}$ are positive definite kernels. When the function $k$ admits such a decomposition, one can choose $k_{+}$ and $k_{-}$ such that the corresponding reproducing kernel Hilbert spaces are disjoint.
\end{theorem}

\subsection{Nystr\"om Method for Reproducing \krein Kernels}
\label{subsec:nystroem}

Let $\mathcal{X}$ be an instance space and $X=\cbrackets{x_1, \dots, x_n}$ an independent sample from a Borel probability measure defined on $\mathcal{X}$. For a positive definite kernel $h$ and a set of landmarks $Z=\cbrackets{z_1, \dots, z_m} \subset \mathcal{X}$, the Nystr\"om method~\citep{Nystrom30,Smola00,Williams01} first projects the evaluation functionals $h\brackets{x_i, \cdot}$ onto $\mathrm{span}\brackets{\cbrackets{h\brackets{z_1, \cdot}, \dots, h\brackets{z_m, \cdot}}}$ and then approximates the kernel matrix $H$ with entries $\cbrackets{H_{ij}=h\brackets{x_i, x_j}}_{i,j=1}^n$ by inner products between the projections of the corresponding evaluation functionals. The projections of the evaluation functionals $h\brackets{x_i, \cdot}$ are linear combinations of the landmarks and these coefficients are given by the following convex optimization problem
\begin{align}
\label{eq:nystroem-pd}
\alpha^{*} = \argmin_{\alpha \in \mathbb{R}^{m \times n}}\  \sum_{i=1}^n \norm{ h\brackets{x_i, \cdot} - \sum_{j=1}^m \alpha_{j,i} h\brackets{z_j, \cdot} }_{\mathcal{H}}^2  .
\end{align} 
While this approach works for positive definite kernels, it cannot be directly applied to reproducing \krein kernels. In particular, let $\mathcal{K}$ be a reproducing kernel \krein space with an indefinite kernel $k\colon \mathcal{X} \times \mathcal{X} \rightarrow \mathbb{R}$. The reproducing \krein kernel $k$ is defined by an indefinite bilinear form $\ip{\cdot}{\cdot}_{\mathcal{K}}$ which does not induce a norm on $\mathcal{K}$ and for all $a, b\in \mathcal{K}$ the value of $\ip{a-b}{a-b}_{\mathcal{K}}$ does not capture the distance. More specifically, as the bilinear form is indefinite then there exists an element $c \in \mathcal{K}$ such that $\ip{c}{c}_{\mathcal{K}}<0$.

For an evaluation functional $k\brackets{x, \cdot} \in \mathcal{K}$ and a linear subspace $\mathcal{L}_Z \subset \mathcal{K}$ spanned by a set of evaluation functionals $\cbrackets{k\brackets{z_1, \cdot}, \dots, k\brackets{z_m, \cdot}}$, we define $\tilde{k}\brackets{x, \cdot}$ to be the orthogonal projection of $k\brackets{x, \cdot}$ onto the subspace $\mathcal{L}_Z$ if the evaluation functional admits a decomposition~\citep{Azizov1981,Iokhvidov82}
\begin{align*}
k\brackets{x, \cdot} = \tilde{k}\brackets{x, \cdot} + k^{\bot}\brackets{x, \cdot} \ ,
\end{align*}
where $\tilde{k}\brackets{x, \cdot} = \sum_{i=1}^{m} \alpha_{i, x} k\brackets{z_i, \cdot}$ with $\alpha_{x}\in \mathbb{R}^m$, and $\ip{k^{\bot}\brackets{x, \cdot}}{\mathcal{L}_Z}_{\mathcal{K}} = 0$. For a landmark $z \in Z$, the inner product between the corresponding evaluation functional $k\brackets{z,\cdot}$ and $k\brackets{x,\cdot}$ then gives
\begin{align}
\label{eq:krein-projection}
k\brackets{x, z} = \ip{k\brackets{x,\cdot}}{k\brackets{z,\cdot}}_{\mathcal{K}} = \sum_{i=1}^m \alpha_{i, x} \tilde{k}\brackets{z_i, z} \ .
\end{align}
Denote with $K_{Z \times Z}$ the block in the kernel matrix $K$ corresponding to landmarks $Z$ and let $k_x=\mathrm{vec}\brackets{k\brackets{x,z_1}, \dots, k\brackets{x,z_m}}$. From Eq.~\ceq{eq:krein-projection} it then follows that $k_x = K_{Z \times Z} \alpha_x$. Thus, in \krein spaces an orthogonal projection exists if the matrix $K_{Z \times Z}$ is non-singular. If this condition is satisfied, then the projection is given by
\begin{align*}
\tilde{k}\brackets{x, \cdot} = \sum_{i=1}^m \alpha_{i, x}^{*} k\brackets{z_i, \cdot} \ \text{with} \  \alpha_x^{*} = K_{Z \times Z}^{-1} k_x \in \mathbb{R}^m \ .
\end{align*}
Having computed the projection of a point onto the span of the landmarks in a \krein space, we now proceed to define the Nystr\"om approximation of the corresponding indefinite kernel matrix. In this, we follow the approach for positive definite kernels~\citep{Scholkopf01,Smola00} and approximate the \krein kernel matrix $K$ using the bilinear form on the span of the landmarks. More specifically, we have that
\begin{align*}
\begin{aligned}
& \tilde{k}\brackets{x_i, x_j} = &&  \ip{\sum_{p=1}^m \alpha_{p,i}^{*} k\brackets{z_p, \cdot}}{\sum_{q=1}^m \alpha_{q,j}^{*} k\brackets{z_q, \cdot}}_{\mathcal{K}} =& \\
&  &&  k_{x_i}^{\top} K_{Z \times Z}^{-1} k_{x_j} \ . &
\end{aligned}
\end{align*}
Thus, the low-rank approximation of the \krein kernel matrix $K$ is given by
\begin{align}
\label{eq:flip-nystroem}
\begin{aligned}
& \tilde{K}_{X \mid Z} = && K_{X \times Z} K_{Z \times Z}^{-1} K_{Z \times Z} K_{Z \times Z}^{-1} K_{Z \times X} = & \\
& && K_{X \times Z} K_{Z \times Z}^{-1} K_{Z \times X} \ . &
\end{aligned}
\end{align}

This approach for low-rank approximation of \krein kernel matrices also provides a direct way for an out-of-sample extension in the non-transductive setting. In particular, for an out-of-sample instance $x \in \mathcal{X}$ we have that if holds
\begin{align*}
\tilde{k}_{x \times X}=\mathrm{vec}(\tilde{k}\brackets{x, x_1}, \dots, \tilde{k}\brackets{x, x_n}) =  K_{X \times Z} K_{Z \times Z}^{-1} k_x  .
\end{align*}

In applications to estimation problems (see Sections~\ref{subsec:krein-regression} and~\ref{subsec:krein-classification}) an approximate low-rank \emph{eigendecomposition} of the kernel matrix, also known as the \emph{one-shot} variant of the Nystr\"om method~\citep[][]{Fowlkes04}, is sometimes preferred over the plain Nystr\"om approximation described above. To derive such a factorization, we first observe that the low-rank approximation of the indefinite kernel matrix can be written as 
\begin{align*}
\tilde{K}_{X \mid Z} =L S L^{\top} \ \text{ with } \ L = K_{X \times Z} U_{Z \times Z} \absolute{D_{Z \times Z}}^{-\frac{1}{2}}   \ ,
\end{align*}
and where $K_{Z \times Z} = U_{Z \times Z} D_{Z \times Z} U_{Z \times Z}^{\top}$ is an eigendecomposition of the block in the kernel matrix corresponding to landmarks and $S=\mathrm{sign}\brackets{D_{Z\times Z}}$. Substituting a singular value decomposition of $L=A \Sigma B^{\top}$ into the latter equation (with orthonormal matrices $A \in \mathbb{R}^{n \times m}$ and $B \in \mathbb{R}^{m \times m}$), we deduce the following low-rank factorization
\begin{align*}
\tilde{K}_{X \mid Z} = A \ \Sigma B^{\top} S B \Sigma \ A^{\top} = A M A^{\top} \ ,
\end{align*}
where $M=\Sigma B^{\top} S B \Sigma$ is a symmetric matrix with an eigendecomposition $M=P\Lambda P^{\top}$. Hence,
\begin{align*}
\tilde{K}_{X \mid Z} = \brackets{AP} \Lambda \brackets{AP}^{\top} \quad \text{with} \quad \brackets{AP}^{\top} AP = \mathbb{I}_m \ .
\end{align*}
As the matrix $\tilde{U}=AP \in \mathbb{R}^{n \times m}$ contains $m$ orthonormal column vectors and $\Lambda$ is a diagonal matrix, we have then derived an approximate low-rank eigendecomposition of $K$.

\subsection{Landmark Selection for the Nystr\"om Method with Indefinite Kernels}
\label{subsec:landmarks}

The effectiveness of a low-rank approximation based on the Nystr\"om method depends crucially on the choice of landmarks and an optimal choice is a difficult discrete optimization problem. The landmark selection for the Nystr\"om method has been studied extensively in the context of approximation of positive definite matrices~\citep[e.g., see][]{Drineas05,Kumar12,Gittens16,Alaoui15,OglicG17icml}. We follow this line of research and present two landmark selection strategies for indefinite \krein kernels inspired by the state-of-the-art sampling schemes: (approximate) kernel $K$-means$++$ sampling~\citep{OglicG17icml} and statistical leverage scores~\citep{Alaoui15,Drineas06,Drineas05}. 

In both cases, we propose to first sample a small number of instances uniformly at random and create a sketch matrix $\tilde{K}=\tilde{U}\Lambda \tilde{U}^{\top}$ using the procedure described in Section~\ref{subsec:nystroem}. Then, using this sketch matrix we propose to approximate: \emph{i}) statistical leverage scores for all instances, and/or \emph{ii}) squared distances between instances in the feature space of the factorization $\tilde{H}=\tilde{L}\tilde{L}^{\top}$ with $\tilde{L}=\tilde{U}\absolute{\Lambda}^{\nicefrac{1}{2}}$. An approximate leverage score assigned to the $i$-th instance is given as the squared norm of the $i$-th row in the matrix $\tilde{U}$, that is $\ell \brackets{x_i}=\Vert \tilde{U} ( i ) \Vert^2$ with $1 \leq i \leq n$. As the two matrices $\tilde{H}$ and $\tilde{K}$ have identical eigenvectors, the approximate leverage scores obtained using the positive definite matrix $\tilde{H}$ capture the informative part of the eigenspace of the indefinite matrix $\tilde{K}$. This follows from the Eckart--Young--Mirsky theorem~\citep{eckart1936,mirsky1960} which implies that the optimal low-rank approximation of an indefinite kernel matrix is given by a set of landmarks spanning the same subspace as that spanned by the eigenvectors corresponding to the top eigenvalues, sorted in descending order with respect to their absolute values. The landmark selection strategy based on the approximate leverage scores then works by taking a set of independent samples from
\begin{align*}
p_{\ell} \brackets{x} = \nicefrac{\ell \brackets{x}}{\sum_{i=1}^n \ell \brackets{x_i}} \ .
\end{align*}
For approximate kernel $K$-means$++$ landmark selection, we propose to perform $K$-means$++$ clustering~\citep{Arthur07} in the feature space defined by the factorization matrix $\tilde{L}$, that is each instance is represented with a row from this matrix. The approach works by first sampling an instance uniformly at random and setting it as the first landmark (i.e., the first cluster centroid). Following this, the next landmark/centroid is selected by sampling an instance with the probability proportional to its clustering contribution. More formally, assuming that the landmarks $\cbrackets{z_1, z_2, \dots, z_s}$ have already been selected the $(s+1)$-st one is selected by taking a sample from the distribution 
\begin{align*}
p_{s+1}^{++}\brackets{x}=\nicefrac{\min_{1\leq i \leq s} \norm{x - z_i}^2}{\sum_{i=1}^n \min_{1\leq j \leq s} \norm{x_i - z_j}^2 } \ .
\end{align*}

\subsection{Scaling Least Squares Methods for Indefinite Kernels using the Nystr\"om Method}
\label{subsec:krein-regression}

We present two regularized risk minimization problems with squared error loss function for scalable learning in reproducing kernel \krein spaces. Our choice of the regularization term is motivated by the considerations in~\citet{OglicG18icml}, where the authors regularize with respect to a decomposition of the \krein kernel into a direct sum of Hilbert spaces. We start with a \krein least squares method (\acro{\krein lsm}) which is a variant of kernel ridge regression, i.e.,
\begin{align*}
\begin{aligned}
 f^{*} = \argmin_{f \in \mathcal{K}} \quad & \frac{1}{n}\sum_{i=1}^n \Brackets{f\brackets{x_i} - y_i}^2 + & \\
& \lambda_{+} \norm{f_{+}}_{+}^2 + \lambda_{-} \norm{f_{-}}_{-}^2 \ , &
\end{aligned}
\end{align*}
where $f=f_{+} \oplus f_{-} \in \mathcal{K}$, $\mathcal{K}=\mathcal{H}_+ \oplus \mathcal{H}_-$ with disjoint $\mathcal{H}_{\pm}$, $f_{\pm} \in \mathcal{H}_{\pm}$, and hyperparameters $\lambda_{\pm} \in \mathbb{R}^{+}$. This is a convex optimization problem for which the representer theorem holds~\citep[][Appendix A]{OglicG18icml} and the optimal solution $f^{*}=\sum_{i=1}^n \alpha_i^{*} k\brackets{x_i, \cdot}$ with $\alpha^{*} \in \mathbb{R}^n$. Applying the reproducing property of the \krein kernel and setting the gradient of the objective to zero, we derive
\begin{align*}
\alpha^{*} = \Brackets{H + n\Lambda_{\pm}}^{-1}Py \ ,
\end{align*}
where $K=UDU^{\top}$, $S=\mathrm{sign}\brackets{D}$, $H=UDSU^{\top}$, $P=USU^{\top}$, and $\Lambda_{\pm}=\lambda_{+}S_{+} + \lambda_{-}\absolute{S_{-}}$ with $S_{\pm}=\nicefrac{\brackets{S \pm \mathbb{I}}}{2}$. 

An important difference compared to stabilization approaches~\citep[e.g., see][]{Loosli16} is that we are solving a regularized risk minimization problem for which a globally optimal solution can be found in polynomial time. Another difference is that stabilization approaches perform subspace descent while we are optimizing jointly over decomposition components of a Krein space. In the special case with $\lambda_+=\lambda_-$, the approach outputs a hypothesis equivalent to that of a stabilization approach along the lines of~\citet{Loosli16}. In particular, the matrix $H$ is called the flip-spectrum transformation of $K$ and $k_{x \times X}^{\top} P$ is the corresponding out-of-sample transformation. Learning with the flip-spectrum transformation of an indefinite kernel matrix was first considered in~\citet{GraepelHBO98} and the corresponding out-of-sample transformation was first proposed in~\citet{Chen09}. The following proposition (a proof is provided in Appendix~\ref{app:proofs}) establishes the equivalence between the least squares method with the flip-spectrum matrix in place of an indefinite kernel matrix and \krein kernel ridge regression regularized with a single hyperparameter.

\begin{restatable}{proposition}{propKreinRegression}
	If the \krein kernel ridge regression problem is regularized via the norm $\norm{\cdot}_{\mathcal{H}_{\mathcal{K}}}$ with $\lambda=\lambda_+=\lambda_-$, then the optimal hypothesis is equivalent to that obtained with kernel ridge regression and the flip-spectrum matrix in place of an indefinite \krein kernel matrix.
\end{restatable}

Having established this, we now proceed to formulate a \krein regression problem with a low-rank approximation $\tilde{K}_{X \mid Z}$ in place of the indefinite kernel matrix $K$. More formally, after substituting the low-rank approximation into \krein kernel ridge regression problem we transform it by
\begin{align*}
\begin{aligned}
& z= \absolute{D_{Z \times Z}}^{-\nicefrac{1}{2}}U_{Z \times Z}^{\top}K_{Z \times X}\alpha = L_{X \mid Z}^{\top}\alpha \ , &\\
& \Phi = K_{X\times Z} U_{Z\times Z} \absolute{D_{Z \times Z}}^{-\nicefrac{1}{2}} S_{Z\times Z} = L_{X \mid Z} S_{Z\times Z} \ , & \\ 
& \tilde{K}_{X \mid Z}\alpha = L_{X \mid Z} S_{Z\times Z} z = \Phi z \ \text{ and }\ \alpha^{\top} H_{\pm} \alpha = z_{\pm}^{\top}z_{\pm} \ ,&
\end{aligned}
\end{align*}
where $H_{\pm}=L_{X \mid Z} \absolute{S_{Z\times Z, \pm}} L_{X \mid Z}^{\top}$, $z_{\pm}=\absolute{S_{Z\times Z, \pm}} z$, and $S_{Z\times Z,\pm}=\nicefrac{\brackets{S_{Z\times Z} \pm \mathbb{I}}}{2}$.
Hence, we can write a low-rank variant of the \krein kernel ridge regression problem as
\begin{align*}
z^{*} = \argmin_{z \in \mathbb{R}^m} \ \norm{\Phi z - y}^2_2 + n\lambda_{+} \norm{z_{+}}^2_2 + n\lambda_{-} \norm{z_{-}}^2_2 \ .
\end{align*}
The problem is convex in $z$ and the optimal solution satisfies
\begin{align*}
z^{*} = \brackets{\Phi^{\top} \Phi + n \Lambda_{\pm}}^{-1} \Phi^{\top} y \ .
\end{align*}
An out-of-sample extension for this learning problem is
\begin{align*}
\tilde{f}^{*}\brackets{x} = k_x^{\top} U_{Z\times Z} \absolute{D_{Z \times Z}}^{-\nicefrac{1}{2}} S_{Z\times Z} z^{*} \ .
\end{align*}

Having introduced a low-rank variant of \krein kernel ridge regression, we proceed to define a scalable variance constrained least squares method (\acro{\krein vc-lsm}). This risk minimization problem is given by~\citep{OglicG18icml}
\begin{align*}
\begin{aligned}
& \min_{f \in \mathcal{K}} && \frac{1}{n} \sum_{i=1}^n \brackets{f\brackets{x_i} - y_i}^2 + \lambda_+ \norm{f_{+}}_{+}^2 + \lambda_{-} \norm{f_{-}}_{-}^2 & \\
& s. t. && \sum_{i=1}^n f\brackets{x_i}^2 = r^2 \ , &
\end{aligned}
\end{align*}
with hyperparameters $r \in \mathbb{R}$ and $\lambda_{\pm} \in \mathbb{R}^{+}$. To simplify our derivations~\citep[just as in][]{OglicG18icml}, we have without loss of generality assumed that the kernel matrix $K$ is centered. Then, the hard constraint fixes the variance of the predictor over training instances. Similar to \krein kernel ridge regression, we can transform this problem into
\begin{align*}
\begin{aligned}
&  z^{*} = && \argmin_{z \in \mathbb{R}^m} \ n\lambda_{+} \norm{z_{+}}^2 + n\lambda_{-} \norm{z_{-}}^2 - 2z^{\top} \Phi^{\top} y & \\
& s.t. && z^{\top} \Phi^{\top} \Phi z = r^2  \ . &
\end{aligned}
\end{align*}
Now, performing a singular value decomposition of $\Phi=A \Delta B^{\top}$ and taking $\gamma=\Delta B^{\top} z$ we obtain
\begin{align*}
\begin{aligned}
& \gamma^{*}= && \argmin_{\gamma \in \mathbb{R}^m} \ n \gamma^{\top} \Delta^{-1}B^{\top} \Lambda_{\pm} B \Delta^{-1} \gamma  -2 (A^{\top}y)^{\top} \gamma & \\
& s.t. &&  \gamma^{\top}\gamma = r^2 \ .& 
\end{aligned}
\end{align*}
A globally optimal solution to this non-convex problem can be computed by following the procedures outlined in~\citet{Gander1989} and~\citet{OglicG18icml}. The cost of computing the solution is $\mathcal{O}\brackets{m^3}$ and the cost for the low-rank transformation of the problem is $\mathcal{O}\brackets{m^3 + m^2n}$. An out-of-sample extension can also be obtained by following the derivation for \krein kernel ridge regression.

\subsection{Scaling Support Vector Machines for Indefinite Kernels using the Nystr\"om Method}
\label{subsec:krein-classification}

In this section, we propose a low-rank support vector machine for scalable classification with indefinite kernels. Our regularization term is again motivated by the considerations in~\citet{OglicG18icml} and that is one of the two main differences compared to \krein support vector machine proposed in~\citet{Loosli16}. The latter approach outputs a hypothesis which can equivalently be obtained using the standard support vector machine with the flip-spectrum kernel matrix combined with the corresponding out-of-sample transformation (introduced in Section~\ref{subsec:krein-regression}). The second difference of our approach compared to~\citet{Loosli16} stems from the fact that in low-rank formulations one optimizes the primal of the problem, defined with the squared hinge loss instead of the plain hinge loss. In particular, the latter loss function is not differentiable and that can complicate the hyperparameter optimization. We note that the identical choice of loss function was used in other works for primal-based optimization of support vector machines~\citep[e.g., see][]{Mangasarian01,Keerthi05}.

We propose the following optimization problem as the \krein squared hinge support vector machine (\acro{\krein sh-svm})
\begin{align*}
\begin{aligned}
 f^{*} = \argmin_{f \in \mathcal{K}} \quad & \frac{1}{n}\sum_{i=1}^n \max \cbrackets{1 - y_if\brackets{x_i}, 0}^2 + & \\
& \lambda_{+} \norm{f_+}_{+}^2 + \lambda_{-} \norm{f_{-}}_{-}^2  \ . &
\end{aligned}
\end{align*}
Similar to Section~\ref{subsec:krein-regression}, the representer theorem holds for this problem and applying the reproducing property of the \krein kernel we can transform it to a matrix form. If we again substitute a low-rank approximation $\tilde{K}_{X \mid Z}$ in place of the \krein kernel matrix $K$, we observe that
\begin{align*}
\begin{aligned}
& f\brackets{x_i} = && \tilde{k}_{x_i}^{\top} \alpha = k_{x_i}^{\top} K_{Z \times Z}^{-1}K_{Z \times X	} \alpha = & \\
& && k_{x_i}^{\top} U_{Z \times Z} \absolute{D_{Z \times Z}}^{-\nicefrac{1}{2}} S_{Z \times Z} z = \Phi_i z \ , &
\end{aligned}
\end{align*}
where $\Phi_i$ denotes the $i$-th row in the matrix $\Phi$. The low-rank variant of the approach can then be written as
\begin{align*}
\begin{aligned}
& z^{*} =  \argmin_{z \in \mathbb{R}^m} && \sum_{i=1}^n \max \cbrackets{1 - y_i \Phi_i z, 0}^2 + & \\
& && n\lambda_{+} \norm{z_+}^2_2 + n\lambda_{-} \norm{z_{-}}^2_2 \ .&
\end{aligned}
\end{align*}

The derivation of the solution follows that for the standard primal-based training of support vector machines with the only difference being that the diagonal matrix $\Lambda_{\pm}$ is used instead of the scalar hyperparameter controlling the hypothesis complexity~\citep[e.g., see][]{Mangasarian01,Keerthi05}. To automatically tune the hyperparameters, one can follow the procedure described in~\citet{ChapelleVBM02} and use implicit derivation to compute the gradient of the optimal solution with respect to the hyperparameters.

\begin{figure*}[t]
	\centering
	\input{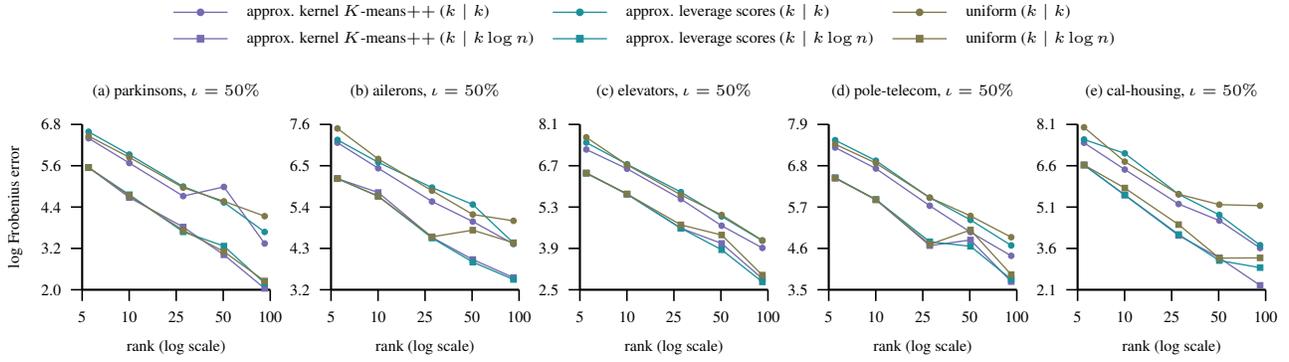}
	\vspace*{-0.8ex}
	\caption{The figure shows the reduction in the approximation error for an indefinite kernel matrix defined as the difference between two Gaussian kernels, which comes as a result of the increase in the approximation rank. In the figure legend, we use $\brackets{k \mid l}$ to express the fact that a rank $k$ approximation of the kernel matrix is computed using a set of $l$ landmarks.}
	\label{fig:delta-gauss-rank-vs-error}\vspace*{1.6ex}
\end{figure*}

\begin{figure*}[!htb]
	\centering
	\begin{tikzpicture}[font=\tiny]
\hypersetup{hidelinks}
\begin{groupplot}[group style={group size=5 by 1, horizontal sep=2.3em}]

\nextgroupplot[
title={(a) parkinsons, $\iota=50\%$},
height=2.2cm,
width=2.5cm,
scale only axis,
axis line style={-},
tick label style={font=\tiny},
label style={font=\tiny},
legend style={font=\tiny, at={(3.0, 1.4)}, anchor=south, draw=none, legend cell align=left, legend columns=3, column sep=.5em},
legend image post style={scale=1.2},
xmin=0.5, xmax=4.5,
ymin=2, ymax=6.8,
axis on top,
xtick={0.5,  1.5,  2.5,  3.5,  4.5},
xticklabels={0.5,  1.5,  2.5,  3.5,  4.5},
ytick={ 2.0 ,  3.2,  4.4,  5.6,  6.8},
yticklabels={ 2.0 ,  3.2,  4.4,  5.6,  6.8},
tick align=outside,
xlabel={log time (Nystr\"om)},
ylabel={log Frobenius error}
]
     
\addplot[c-km, only marks, mark=o, mark size=2pt]
coordinates {
	( 0.70295901 ,  6.39728207 )
	( 0.78681822 ,  5.67832405 )
	( 0.80589484 ,  4.71918142 )
	(   0.92237502 ,  4.98599623 )
	(1.1850326 , 3.34220906)
};
\addlegendentry{approx. kernel $K$-means$++$ ($k \mid k$)};
        
\addplot[c-lss-1, only marks, mark=o, mark size=2pt]
coordinates {
	( 0.78623862 ,  6.58661312 )
	( 0.77509552 ,  5.92494051 )
	(  0.83931027 ,  4.99667569 )
	(  0.91141035 ,  4.53500358 )
	(1.10946096 , 3.67729154)
};
\addlegendentry{approx. leverage scores ($k \mid k$)};
       
\addplot[c-dpp-1, only marks, mark=o, mark size=2pt]
coordinates {
	( 0.66125996 ,  6.46085002 )
	(  0.6281714 ,  5.84395709 )
	(  0.74398442 ,  4.96470334 )
	( 0.90783116  ,  4.56802462 )
	(  0.97567741,  4.13635335)
};
\addlegendentry{uniform ($k \mid k$)};
 
\addplot[c-km, only marks, mark=square, mark size=2pt]
coordinates {
	( 2.37061955 ,  5.55140318 )
	( 2.13755006 ,  4.67755246 )
	( 2.28798736 ,  3.82063305 )
	(  3.19250434 ,  3.01414714 )
	( 4.2114715, 2.03976865)
};
\addlegendentry{approx. kernel $K$-means$++$ ($k \mid k \log n$)};

\addplot[c-lss-1, only marks, mark=square, mark size=2pt]
coordinates {
	(2.45425857 ,  5.54143307 )
	(2.09268057  ,  4.75882497 )
	(  2.19443765 ,  3.68181224 )
	( 2.82431615 ,  3.26775237 )
	( 4.00683094, 2.19055765)
};
\addlegendentry{approx. leverage scores ($k \mid k \log n$)};

\addplot[c-dpp-1, only marks, mark=square, mark size=2pt]
coordinates {
	( 2.21118729 , 5.55047165 )
	( 1.95605602 ,  4.72641063 )
	( 2.14945467  ,  3.7265337 )
	( 2.8197973 ,  3.10680721 )
	(  3.82641001,  2.25395925)
};
\addlegendentry{uniform ($k \mid k \log n$)};

\nextgroupplot[
title={(b) ailerons, $\iota=50\%$},
height=2.2cm,
width=2.5cm,
scale only axis,
axis line style={-},
tick label style={font=\tiny},
label style={font=\tiny},
legend style={font=\tiny},
xmin=0.5, xmax=3.7,
ymin=3.2, ymax=7.6,
axis on top,
xtick={0.5,  1.3,  2.1,  2.9,  3.7},
xticklabels={0.5,  1.3,  2.1,  2.9,  3.7},
ytick={  3.2,  4.3,  5.4,  6.5,  7.6},
yticklabels={  3.2,  4.3,  5.4,  6.5,  7.6},
tick align=outside,
xlabel={log time (Nystr\"om)}
]

\addplot[c-km, only marks, mark=o, mark size=2pt]
coordinates {
	( 0.647908 ,  7.11378931 )
	(  0.61025722 ,  6.4339291 )
	(  0.71207827  ,  5.54797731 )
	( 0.8749465  ,  5.01528408 )
	( 1.14649539, 4.41645038)
};
       
\addplot[c-lss-1, only marks, mark=o, mark size=2pt]
coordinates {
	( 0.64987622 , 7.18801607 )
	( 0.66238921,  6.5915836 )
	( 0.68829009 ,   5.91910503 )
	(  0.82997596 , 5.46938348  )
	(  1.07822199,  4.44316666)
};
      
\addplot[c-dpp-1, only marks, mark=o, mark size=2pt]
coordinates {
	( 0.54048494  ,  7.4908546 )
	(  0.56108442 ,  6.68153955 )
	( 0.64648183 ,  5.83546801 )
	(  0.80000101 , 5.20344945 )
	(  0.95637571,  5.03419612)
};

\addplot[c-km, only marks, mark=square, mark size=2pt]
coordinates {
	( 0.97645396 ,  6.15624439 )
	( 1.26130572 ,  5.78780006 )
	( 1.93771194 ,  4.59327317 )
	(  2.53982073 ,  4.00399021 )
	( 3.62054915, 3.52295417 )
};
 
\addplot[c-lss-1, only marks, mark=square, mark size=2pt]
coordinates {
	( 0.97684218 ,  6.1571411 )
	( 1.22456947 ,  5.68895159 )
	( 1.93034658 ,   4.57680589 )
	( 2.46962913 ,  3.93621284 )
	(  3.43269106,  3.47386001)
};

\addplot[c-dpp-1, only marks, mark=square, mark size=2pt]
coordinates {
	( 0.96866224 ,   6.15683105 )
	( 1.21161215 ,  5.68646381  )
	( 1.8395972 , 4.60863931  )
	(  2.41283341  ,  4.78545861 )
	( 3.42267903,  4.4519801)
};

\nextgroupplot[
title={(c) elevators, $\iota=50\%$},
height=2.2cm,
width=2.5cm,
scale only axis,
axis line style={-},
tick label style={font=\tiny},
label style={font=\tiny},
legend style={font=\tiny},
xmin=0.7, xmax=4.3,
ymin=2.5, ymax=8.1,
axis on top,
xtick={ 0.7,  1.6,  2.5,  3.4,  4.3},
xticklabels={ 0.7,  1.6,  2.5,  3.4,  4.3},
ytick={ 2.5,  3.9,  5.3,  6.7,  8.1},
yticklabels={ 2.5,  3.9,  5.3,  6.7,  8.1},
tick align=outside,
xlabel={log time (Nystr\"om)}
]
         
\addplot[c-km, only marks, mark=o, mark size=2pt]
coordinates {
	(  0.896512 ,  7.25030994 )
	( 0.90626842 , 6.59786301  )
	(  0.96109688 , 5.57140522  )
	( 1.09335068 ,  4.66854042 )
	( 1.29819244,  3.91586563)
};
          
\addplot[c-lss-1, only marks, mark=o, mark size=2pt]
coordinates {
	( 0.9812415 ,  7.48515933 )
	( 0.92271441 ,   6.75048263 )
	(  0.9721911 ,   5.80744404 )
	(  1.08036434 ,   4.98409929 )
	(1.24581877,  4.15892856 )
};
          
\addplot[c-dpp-1, only marks, mark=o, mark size=2pt]
coordinates {
	( 0.88409942 ,  7.66241029 )
	(  0.830805 ,  6.73263622 )
	( 0.91107773 ,  5.71123736 )
	( 1.01882338 ,  5.03208369 )
	( 1.14835335, 4.17408522)
};
   
\addplot[c-km, only marks, mark=square, mark size=2pt]
coordinates {
	(  1.45371515 ,   6.44852554 )
	( 2.03803031  ,  5.73791276 )
	( 2.19893469 ,   4.58224266 )
	(  2.71011482 ,  4.06922663 )
	( 4.19608148, 2.87691023)
};

\addplot[c-lss-1, only marks, mark=square, mark size=2pt]
coordinates {
	( 1.45797896 ,   6.4501965 )
	( 2.01569685 ,  5.74201547 )
	( 2.19784542  , 4.58590563)
	( 2.59538856 ,   3.85480962)
	(  3.96378265,   2.76641589)
};
 
\addplot[c-dpp-1, only marks, mark=square, mark size=2pt]
coordinates {
	( 1.42347457 , 6.45389789 )
	( 1.95740069 ,  5.74283199 )
	( 2.12808893, 4.69699246 )
	( 2.56511358 ,   4.35365756 )
	( 3.91995223, 2.99710113)
};

\nextgroupplot[
title={(d) pole-telecom, $\iota=50\%$},
height=2.2cm,
width=2.5cm,
scale only axis,
axis line style={-},
tick label style={font=\tiny},
label style={font=\tiny},
legend style={font=\tiny},
xmin=0.5, xmax=4.5,
ymin=3.5, ymax=7.9,
axis on top,
xtick={0.5,  1.5,  2.5,  3.5,  4.5},
xticklabels={0.5,  1.5,  2.5,  3.5,  4.5},
ytick={ 3.5,  4.6,  5.7,  6.8,  7.9},
yticklabels={ 3.5,  4.6,  5.7,  6.8,  7.9},
tick align=outside,
xlabel={log time (Nystr\"om)}
]
        
\addplot[c-km, only marks, mark=o, mark size=2pt]
coordinates {
	( 0.72295901  ,  7.28435882 )
	(  0.79681822  ,   6.72481232 )
	( 0.81589484 ,   5.73474204 )
	( 0.91237502 ,   5.03277957 )
	( 1.1850326,  4.40257979)
};
      
\addplot[c-lss-1, only marks, mark=o, mark size=2pt]
coordinates {
	( 0.78623862 ,  7.48014009)
	( 0.77509552 ,   6.93420283  )
	(  0.83931027 , 5.94862556)
	(  0.91141035  ,  5.35814911 )
	(  1.10946096,  4.67948334)
};
         
\addplot[c-dpp-1, only marks, mark=o, mark size=2pt]
coordinates {
	( 0.66125996 ,  7.37332184 )
	( 0.62817149 ,   6.86874282 )
	( 0.74398442 ,  5.9514898 )
	( 0.90783116  ,   5.46522194 )
	( 0.97567741,  4.89860851 )
};

\addplot[c-km, only marks, mark=square, mark size=2pt]
coordinates {
	( 1.37061955 ,   6.47272712)
	( 1.93755006   , 5.89745088  )
	( 2.18798736 ,  4.67475204  )
	(  3.19250434  ,  4.82363523 )
	(  4.21147152,  3.71738401 )
};

\addplot[c-lss-1, only marks, mark=square, mark size=2pt]
coordinates {
	( 1.45425857 ,   6.48112271 )
	( 2.09268057 ,  5.90003086 )
	( 2.19443765 ,   4.77770668 )
	( 2.82431615 ,   4.65590529 )
	(4.00683094,  3.77612286)
};
 
\addplot[c-dpp-1, only marks, mark=square, mark size=2pt]
coordinates {
	( 1.21118729 ,   6.47019571 )
	( 1.95605602 ,  5.89518732 )
	(  2.14945467 , 4.70788597  )
	( 2.8197973 ,  5.08993763 )
	( 3.82641001,  3.90310198 )
};

\nextgroupplot[
title={(e) cal-housing, $\iota=50\%$},
height=2.2cm,
width=2.5cm,
scale only axis,
axis line style={-},
tick label style={font=\tiny},
label style={font=\tiny},
legend style={font=\tiny},
xmin=1.5, xmax=4.3,
ymin=2.1, ymax=8.1,
axis on top,
xtick={ 1.5,  2.2,  2.9,  3.6,  4.3},
xticklabels={ 1.5,  2.2,  2.9,  3.6,  4.3},
ytick={2.1,  3.6,  5.1,  6.6,  8.1},
yticklabels={2.1,  3.6,  5.1,  6.6,  8.1},
tick align=outside,
xlabel={log time (Nystr\"om)}
]
          
\addplot[c-km, only marks, mark=o, mark size=2pt]
coordinates {
	( 1.73946311 ,  7.43285572 )
	(  1.71606227 ,   6.4629178 )
	( 2.10291659 ,  5.21070416 )
	(  2.07198797 ,   4.6097437 )
	( 2.475664 ,  3.61308444)
};
        
\addplot[c-lss-1, only marks, mark=o, mark size=2pt]
coordinates {
	( 1.7659463 , 7.55261577 )
	(  1.66797886 ,   7.047875 )
	( 2.10490894 ,   5.5656471  )
	( 2.01752357 ,  4.81665998 )
	(  2.47972229, 3.71701847)
};
          
\addplot[c-dpp-1, only marks, mark=o, mark size=2pt]
coordinates {
	( 1.6689394,   7.99382575  )
	(  1.6166923 ,   6.74845117  )
	( 2.0204552 ,  5.56681481  )
	( 2.00008158 ,   5.1885259  )
	(  2.46434668,  5.15130695)
};
 
\addplot[c-km, only marks, mark=square, mark size=2pt]
coordinates {
	( 2.16167435 ,   6.62591122 )
	( 1.93469978 ,  5.53442799  )
	(  2.49153872 ,  4.0695395 )
	(  3.30716426 ,   3.25743981 )
	( 3.82677752,   2.25943898)
};
 
\addplot[c-lss-1, only marks, mark=square, mark size=2pt]
coordinates {
	( 2.15167181 ,  6.62618657 )
	( 1.97921136 ,  5.54148874)
	( 2.50181436 ,  4.10480243 )
	( 3.26619109 ,    3.16004654 )
	(  3.71192685,  2.90245776)
};
 
\addplot[c-dpp-1, only marks, mark=square, mark size=2pt]
coordinates {
	( 1.5524493 ,  6.62773368 )
	( 1.89038937 , 5.79085792 )
	( 2.34816479 ,  4.46345036 )
	(  3.13301416 ,  3.24832749 )
	( 3.67940162,   3.25631369)
};

\end{groupplot}
\end{tikzpicture}
	\vspace*{-0.8ex}
	\caption{The figure shows the approximation errors of Nystr\"om low-rank approximations (with different approximation ranks) as a function of time required to compute these approximations. In the figure legend, we again use $\brackets{k \mid l}$ to express the fact that a rank $k$ approximation of the kernel matrix is computed using a set of $l$ landmarks.}
	\label{fig:delta-gauss-time-vs-error}
\end{figure*}
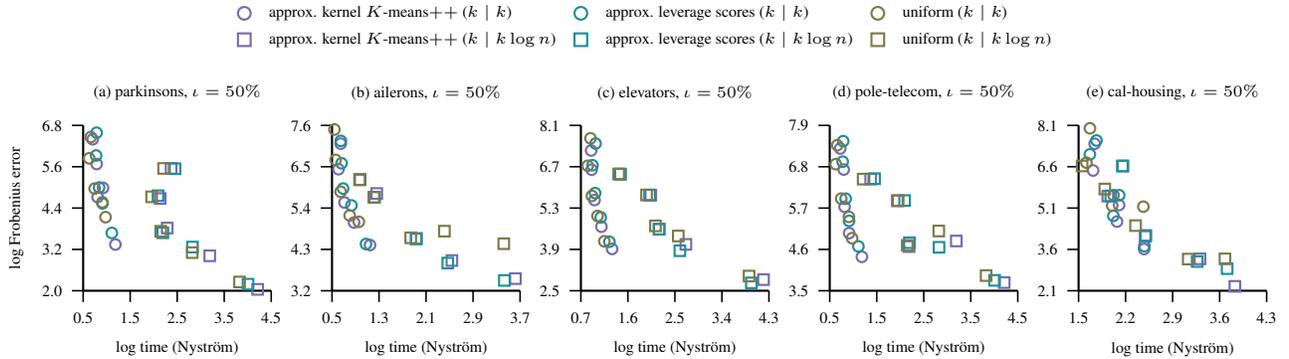

We conclude with a discussion of a potential shortcoming inherent to the \krein support vector machine~\citep{Loosli16}. As the following proposition shows (a proof can be found in Appendix~\ref{app:proofs}), that approach is equivalent to the standard support vector machine with the flip-spectrum matrix in place of an indefinite \krein kernel matrix~\citep{GraepelHBO98}, combined with the corresponding out-of-sample transformation~\citep{Chen09}. 

\begin{restatable}{proposition}{propKreinSVM}
	\label{prop:krein-svm}
	The \krein support vector machine~\citep{Loosli16} is equivalent to the standard support vector machine with the flip-spectrum matrix in place of an indefinite \krein kernel matrix~\citep{GraepelHBO98}, combined with the out-of-sample transformation from~\citet{Chen09}.
\end{restatable}

While at the first glance this result seems incremental, it makes an important contribution towards understanding the \krein support vector machines~\citep{Loosli16}. In particular, the discussion of experiments in~\citet{Loosli16} differentiates between the \krein support vector machines and the flip-spectrum approach. This happens despite the illustration indicating that they produce identical hypotheses in synthetic experiments~\citep[e.g., see Figures $3$ and $4$ in][and the discussion therein]{Loosli16}.

\begin{figure*}[!htb]
	\centering
	\begin{tikzpicture}[font=\tiny]
\hypersetup{hidelinks}
\begin{groupplot}[group style={group size=5 by 1, horizontal sep=2.3em}]

\nextgroupplot[
title={(a) coilyork, $\iota=26\%$},
height=2.2cm,
width=2.5cm,
scale only axis,
axis line style={-},
tick label style={font=\tiny},
label style={font=\tiny},
legend style={font=\tiny, at={(3.0, 1.4)}, anchor=south, draw=none, legend cell align=left, legend columns=4, column sep=2em},
legend image post style={scale=1.2},
xmin=1.5, xmax=4.7,
ymin=30, ymax=50,
axis on top,
xtick={1.5,  2.3,  3.1,  3.9,  4.7},
xticklabels={5, 10,  25,  50,  100},
ytick={30.,  35.,  40.,  45.,  50},
yticklabels={ 30,  35,  40,  45,  50},
tick align=outside,
xlabel={rank (log scale)},
ylabel={classification error (\%)}
]

\addplot [c-km, solid, line width=0.5pt, mark=*, mark size=1pt, mark repeat={1}, mark phase=1]
coordinates {
( 1.60943791243 , 48.275862069 )
( 2.30258509299 , 44.6428571429 )
( 3.21887582487 , 43.0952380952 )
( 3.91202300543 , 39.7619047619 )
( 4.60517018599 , 32.7380952381 )
};
\addlegendentry{\acro{sf-lsm}};

\addplot [c-lss-1, solid, line width=0.5pt, mark=square*, mark size=1pt, mark repeat={1}, mark phase=1]
coordinates {
( 1.60943791243 , 37.5 )
( 2.30258509299 , 42.118226601 )
( 3.21887582487 , 40.2709359606 )
( 3.91202300543 , 36.1904761905 )
( 4.60517018599 , 33.3128078818 )
};
\addlegendentry{\krein \acro{lsm}};

\addplot [c-dpp-1, solid, line width=0.5pt, mark=diamond*, mark size=1.5pt, mark repeat={1}, mark phase=1]
coordinates {
( 1.60943791243 , 36.8226600985 )
( 2.30258509299 , 42.118226601 )
( 3.21887582487 , 40.2709359606 )
( 3.91202300543 , 37.2988505747 )
( 4.60517018599 , 32.183908046 )
};
\addlegendentry{\krein \acro{vc-lsm}};

\addplot [lloyd-km, solid, line width=0.5pt, mark=triangle*, mark size=1.5pt, mark repeat={1}, mark phase=1]
coordinates {
( 1.60943791243 , 46.5517241379 )
( 2.30258509299 , 47.3522167488 )
( 3.21887582487 , 43.1034482759 )
( 3.91202300543 , 40.3325123153 )
( 4.60517018599 , 38.6083743842 )
};
\addlegendentry{\krein \acro{sh-svm}};

\nextgroupplot[
title={(b) balls3D, $\iota=0.07\%$},
height=2.2cm,
width=2.5cm,
scale only axis,
axis line style={-},
tick label style={font=\tiny},
label style={font=\tiny},
legend style={font=\tiny, at={(3.0, 1.4)}, anchor=south, draw=none, legend cell align=left, legend columns=3, column sep=2em},
legend image post style={scale=1.2},
xmin=1.5, xmax=4.7,
ymin=0, ymax=40,
axis on top,
xtick={1.5,  2.3,  3.1,  3.9,  4.7},
xticklabels={5, 10,  25,  50,  100},
ytick={0.,  10.,  20.,  30.,  40.},
yticklabels={ 0,  10,  20,  30,  40},
tick align=outside,
xlabel={rank (log scale)}
]

\addplot [c-km, solid, line width=0.5pt, mark=*, mark size=1pt, mark repeat={1}, mark phase=1]
coordinates {
( 1.60943791243 , 27.5 )
( 2.30258509299 , 10.0 )
( 3.21887582487 , 5.0 )
( 3.91202300543 , 5.0 )
( 4.60517018599 , 0.0 )
};

\addplot [c-lss-1, solid, line width=0.5pt, mark=square*, mark size=1pt, mark repeat={1}, mark phase=1]
coordinates {
( 1.60943791243 , 40.0 )
( 2.30258509299 , 30.0 )
( 3.21887582487 , 7.5 )
( 3.91202300543 , 0.0 )
( 4.60517018599 , 0.0 )
};

\addplot [c-dpp-1, solid, line width=0.5pt, mark=diamond*, mark size=1.5pt, mark repeat={1}, mark phase=1]
coordinates {
( 1.60943791243 , 35.0 )
( 2.30258509299 , 35.0 )
( 3.21887582487 , 5.0 )
( 3.91202300543 , 0.0 )
( 4.60517018599 , 0.0 )
};

\addplot [lloyd-km, solid, line width=0.5pt, mark=triangle*, mark size=1.5pt, mark repeat={1}, mark phase=1]
coordinates {
( 1.6094379124341003 , 35.0 )
( 2.302585092994046 , 15.0 )
( 3.2188758248682006 , 0.0 )
( 3.912023005428146 , 0.0 )
( 4.605170185988092 , 0.0 )
};

\nextgroupplot[
title={(c) prodom, $\iota=99\%$},
height=2.2cm,
width=2.5cm,
scale only axis,
axis line style={-},
tick label style={font=\tiny},
label style={font=\tiny},
legend style={font=\tiny},
xmin=1.5, xmax=4.7,
ymin=0, ymax=20,
axis on top,
xtick={1.5,  2.3,  3.1,  3.9,  4.7},
xticklabels={5, 10,  25,  50,  100},
ytick={   0.,   5.,  10.,  15.,  20},
yticklabels={ 0,  5,  10,  15,  20},
tick align=outside,
xlabel={rank (log scale	)}
]

\addplot [c-km, solid, line width=0.5pt, mark=*, mark size=1.pt, mark repeat={1}, mark phase=1]
coordinates {
( 1.60943791243 , 19.422624595 )
( 2.30258509299 , 11.1324786325 )
( 3.21887582487 , 3.44827586207 )
( 3.91202300543 , 5.00965250965 )
( 4.60517018599 , 1.5384842971 )
};

\addplot [c-lss-1, solid, line width=0.5pt, mark=square*, mark size=1pt, mark repeat={1}, mark phase=1]
coordinates {
( 1.60943791243 , 14.4225506294 )
( 2.30258509299 , 15.162835249 )
( 3.21887582487 , 4.03035661656 )
( 3.91202300543 , 1.5384842971 )
( 4.60517018599 , 0.960813029779 )
};

\addplot [c-dpp-1, solid, line width=0.5pt, mark=diamond*, mark size=1.5pt, mark repeat={1}, mark phase=1]
coordinates {
( 1.60943791243 , 14.4225506294 )
( 2.30258509299 , 14.9712643678 )
( 3.21887582487 , 4.03035661656 )
( 3.91202300543 , 1.5384842971 )
( 4.60517018599 , 0.960813029779 )
};

\addplot [lloyd-km, solid, line width=0.5pt, mark=triangle*, mark size=1.5pt, mark repeat={1}, mark phase=1]
coordinates {
( 1.60943791243 , 19.7318007663 )
( 2.30258509299 , 13.0785958372 )
( 3.21887582487 , 2.88981288981 )
( 3.91202300543 , 0.960813029779 )
( 4.60517018599 , 0.0 )
};

\nextgroupplot[
title={(d) chicken10-120, $\iota=18\%$},
height=2.2cm,
width=2.5cm,
scale only axis,
axis line style={-},
tick label style={font=\tiny},
label style={font=\tiny},
legend style={font=\tiny},
xmin=1.5, xmax=4.7,
ymin=13, ymax=37,
axis on top,
xtick={1.5,  2.3,  3.1,  3.9,  4.7},
xticklabels={5, 10,  25,  50,  100},
ytick={  13.,  19.,  25.,  31.,  37.},
yticklabels={    13,  19,  25,  31,  37},
tick align=outside,
xlabel={rank (log scale	)}
]

\addplot [c-km, solid, line width=0.5pt, mark=*, mark size=1.pt, mark repeat={1}, mark phase=1]
coordinates {
( 1.60943791243 , 35.5555555556 )
( 2.30258509299 , 25.8333333333 )
( 3.21887582487 , 20.0 )
( 3.91202300543 , 14.4444444444 )
( 4.60517018599 , 14.595959596 )
};

\addplot [c-lss-1, solid, line width=0.5pt, mark=square*, mark size=1pt, mark repeat={1}, mark phase=1]
coordinates {
( 1.60943791243 , 24.696969697 )
( 2.30258509299 , 22.4747474747 )
( 3.21887582487 , 18.8888888889 )
( 3.91202300543 , 17.7777777778 )
( 4.60517018599 , 15.7070707071 )
};

\addplot [c-dpp-1, solid, line width=0.5pt, mark=diamond*, mark size=1.5pt, mark repeat={1}, mark phase=1]
coordinates {
( 1.60943791243 , 23.5858585859 )
( 2.30258509299 , 22.2222222222 )
( 3.21887582487 , 19.0909090909 )
( 3.91202300543 , 17.9797979798 )
( 4.60517018599 , 14.7727272727 )
};

\addplot [lloyd-km, solid, line width=0.5pt, mark=triangle*, mark size=1.5pt, mark repeat={1}, mark phase=1]
coordinates {
( 1.6094379124341003 , 27.7777777778 )
( 2.302585092994046 , 22.2222222222 )
( 3.2188758248682006 , 17.7777777778 )
( 3.912023005428146 , 17.9797979798 )
( 4.605170185988092 , 14.4444444444 )
};

\nextgroupplot[
title={(e) protein, $\iota=0.07\%$},
height=2.2cm,
width=2.5cm,
scale only axis,
axis line style={-},
tick label style={font=\tiny},
label style={font=\tiny},
legend style={font=\tiny},
xmin=1.5, xmax=4.7,
ymin=3, ymax=35,
axis on top,
xtick={1.5,  2.3,  3.1,  3.9,  4.7},
xticklabels={5, 10,  25,  50,  100},
ytick={  3.,  11.,  19.,  27.,  35},
yticklabels={  3,  11,  19,  27,  35},
tick align=outside,
xlabel={rank (log scale)}
]

\addplot [c-km, solid, line width=0.5pt, mark=*, mark size=1pt, mark repeat={1}, mark phase=1]
coordinates {
( 1.60943791243 , 33.3333333333 )
( 2.30258509299 , 14.2857142857 )
( 3.21887582487 , 14.2857142857 )
( 3.91202300543 , 6.92640692641 )
( 4.60517018599 , 4.65367965368 )
};

\addplot [c-lss-1, solid, line width=0.5pt, mark=square*, mark size=1pt, mark repeat={1}, mark phase=1]
coordinates {
( 1.60943791243 , 18.6147186147 )
( 2.30258509299 , 19.0476190476 )
( 3.21887582487 , 13.6363636364 )
( 3.91202300543 , 4.65367965368 )
( 4.60517018599 , 4.65367965368 )
};
      
\addplot [c-dpp-1, solid, line width=0.5pt, mark=diamond*, mark size=1.5pt, mark repeat={1}, mark phase=1]
coordinates {
( 1.60943791243 , 18.6147186147 )
( 2.30258509299 , 19.0476190476 )
( 3.21887582487 , 13.6363636364 )
( 3.91202300543 , 4.65367965368 )
( 4.60517018599 , 4.7619047619 )
};

\addplot [lloyd-km, solid, line width=0.5pt, mark=triangle*, mark size=1.5pt, mark repeat={1}, mark phase=1]
coordinates {
( 1.60943791243 , 7.14285714286 )
( 2.30258509299 , 7.03463203463 )
( 3.21887582487 , 9.52380952381 )
( 3.91202300543 , 4.7619047619 )
( 4.60517018599 , 4.7619047619 )
};

\end{groupplot}
\end{tikzpicture}
	\vspace*{-0.8ex}
	\caption{The figure shows the reduction in the classification error as the approximation rank of a Nystr\"om low-rank approximation increases. The reported error is the median classification error obtained using $10$-fold stratified cross-validation.}
	\label{fig:dissim-rank-vs-error}
\end{figure*}
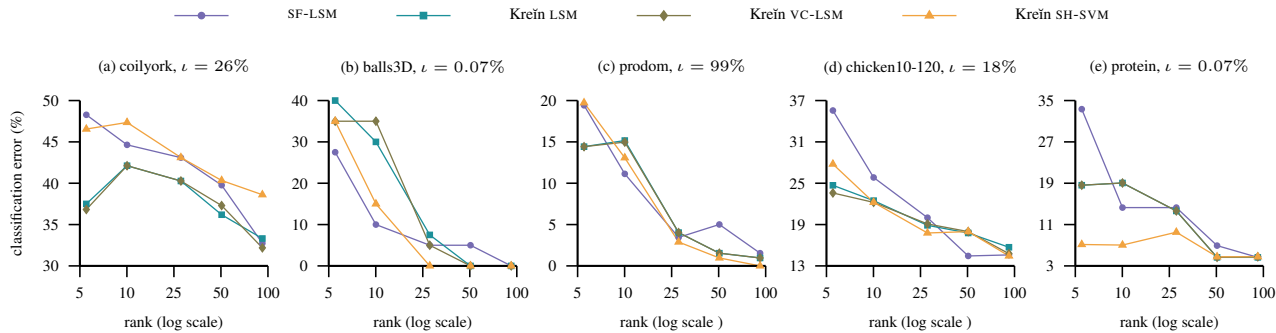

\section{Experiments}
\label{sec:experiments}

In this section, we report the results of experiments aimed at demonstrating the effectiveness of: \emph{i}) the Nystr\"om method in low-rank approximations of indefinite kernel matrices, and \emph{ii}) the described scalable \krein approaches in classification tasks with pairwise (dis)similarity matrices.

In the first set of experiments, we take several datasets from \acro{uci} and \acro{liacc} repositories and define kernel matrices on them using the same indefinite kernels as previous work~\citep[][Appendix D]{OglicG18icml}. We use 
\begin{align*}
0 \leq \iota=\nicefrac{\sum_{\cbrackets{i \colon \lambda_i < 0}} \absolute{\lambda_i} } {\sum_i \absolute{\lambda_i}} \leq 1
\end{align*}
to quantify the level of indefiniteness of a kernel matrix. Prior to computation of kernel matrices, all the data matrices were normalized to have mean zero and unit variance across features. Following this, we have applied the Nystr\"om method with landmark selection strategies presented in Section~\ref{subsec:landmarks} to derive approximations with different ranks. We measure the effectiveness of a low-rank approximation with its error in the Frobenius norm. To quantify the effectiveness of the approximate eigendecomposition of the kernel matrix (i.e., the one-shot Nystr\"om method) derived in Section~\ref{subsec:nystroem}, we have performed rank $k$ approximations using sets of $k\log n$ landmarks. Figures~\ref{fig:delta-gauss-rank-vs-error} and~\ref{fig:delta-gauss-time-vs-error} summarize the results obtained with an indefinite kernel defined by the difference between two Gaussian kernels. The reported error/time is the median error/time over $10$ repetitions of the experiment. Figure~\ref{fig:delta-gauss-rank-vs-error} indicates a sharp (approximately exponential) decay in the approximation error as the rank of the approximation increases. The devised approximate kernel $K$-means$++$ sampling strategy performs the best in terms of the accuracy in the experiments where rank $k$ approximations are generated using $k$ landmarks. The approximate leverage score strategy is quite competitive and in rank $k$ approximations generated using $k\log n$ landmarks it performs as good or even better than the approximate kernel $K$-means$++$ sampling scheme.~\citet{OglicG17icml} have evaluated these two state-of-the-art strategies on the low-rank approximation of positive definite kernels. In contrast to that work, we had to resort to an approximate kernel $K$-means$++$ sampling scheme because of the indefiniteness of the bilinear form defining a \krein space. As a result of this, we can observe the lack of a gap between the curves describing the two sampling strategies, compared to the results reported in~\citet{OglicG17icml} for positive definite kernels. Our hypothesis is that this is due to sub-optimal choices of landmarks that define sketch matrices. In our simulations, we have generated sketches by sampling the corresponding landmarks uniformly at random. In support of this hypothesis, rather large approximation errors for uniformly selected landmarks in approximation of other indefinite kernels can be observed (see Appendix~\ref{app:plots}). Figure~\ref{fig:delta-gauss-time-vs-error} reports the time required to generate a Nystr\"om low-rank approximation and indicates that the considered sampling strategies amount to only a small fraction of the total time required to generate the low-rank approximation.

\begin{table*}[!htb]
	\centering
	\fontsize{8}{10}\selectfont
\setlength{\tabcolsep}{3pt}
\begin{tabular}{l|l|rr|rr|rr|rr}

\hline

\hline

\multirow{2}{*}{\sc Dataset} & 
\multirow{2}{*}{\sc Dissimilarity type} &
\multicolumn{8}{c}{\sc Rank $100$ Approximation}  \\\cline{3-10}

& & 
\multicolumn{2}{c|}{\textsc{\krein vc-lsm}} &
\multicolumn{2}{c|}{\textsc{\krein lsm}} &
\multicolumn{2}{c|}{\textsc{\krein sh-svm}} &
\multicolumn{2}{c}{\textsc{sf-lsm}}  \\

\hline

\hline

coilyork & 
Graph matching &
$32.22 $ & $(\pm 7.89)$ &
$\textbf{31.21} $ & $(\pm 5.28)$ & 
$38.20 $ & $(\pm 7.20)$ &
$35.33 $ & $(\pm 10.09)$  \\

balls 3D & 
Shortest distance between balls &
$1.00 $ & $(\pm 2.00)$ &
$0.50 $ & $(\pm 1.50)$ &
$\textbf{0.00} $ & $(\pm 0.00)$ &
$0.50 $ & $(\pm 1.50)$ \\

prodom & 
Structural alignment of proteins &
$0.92 $ & $(\pm 0.46)$ &
$0.92 $ & $(\pm 0.46)$ & 
$\textbf{0.54} $ & $(\pm 0.47)$ &
$1.57 $ & $(\pm 0.58)$ \\

chicken10 & 
String edit distance &
$16.35 $ & $(\pm 4.31)$ &
$15.69 $ & $(\pm 4.97)$ & 
$16.82 $ & $(\pm 6.57)$ &
$\textbf{14.37} $ & $(\pm 4.02)$  \\

protein & 
Structural alignment of proteins &
$4.19 $ & $(\pm 2.47)$ &
$\textbf{3.72} $ & $(\pm 2.76)$ & 
$5.23 $ & $(\pm 2.89)$ & 
$5.15 $ & $(\pm 3.91)$ \\

zongker & 
Deformable template matching &
$17.70 $ & $(\pm 2.06)$ &
$17.75 $ & $(\pm 2.23)$ & 
$\textbf{15.30} $ & $(\pm 3.39)$ &
$17.05 $ & $(\pm 2.36)$\\

chicken25 & 
String edit distance &
$19.29 $ & $(\pm 4.64)$ &
$20.41 $ & $(\pm 4.09)$ & 
$25.77 $ & $(\pm 4.68)$ &
$\textbf{18.17} $ & $(\pm 6.67)$ \\

pdish57 & 
Hausdorff distance &
$3.40 $ & $(\pm 0.39)$ &
$3.40 $ & $(\pm 0.42)$ & 
$\textbf{2.73} $ & $(\pm 0.62)$ &
$3.03 $ & $(\pm 0.67)$  \\

pdism57 & 
Hausdorff distance &
$0.38 $ & $(\pm 0.26)$ &
$0.38 $ & $(\pm 0.26)$ & 
$\textbf{0.30} $ & $(\pm 0.29)$ &
$0.63 $ & $(\pm 0.42)$ \\

woody50 & 
Plant leaves' shape dissimilarity &
$30.84 $ & $(\pm 5.25)$ &
$30.47 $ & $(\pm 5.54)$ & 
$38.42 $ & $(\pm 7.13)$ &
$\textbf{26.41} $ & $(\pm 4.42)$ \\

\hline

\hline

\end{tabular}
	\caption{The table reports the results of our experiment on benchmark datasets for learning with indefinite kernels~\citep{Pekalska05}. The goal of the experiment is to evaluate the effectiveness of the state-of-the-art approaches for scalable learning in reproducing kernel \krein spaces on classification tasks with pairwise dissimilarity matrices. We measure the effectiveness of an approach using the average classification error obtained using $10$-fold stratified cross-validation (standard deviations are given in the brackets).}
	\label{tbl:baseline-results}
\end{table*}

In the second set of experiments, we evaluate the effectiveness of the proposed least square methods and the support vector machine on classification tasks\footnote{\tiny \url{http://prtools.org/disdatasets/index.html}} with pairwise dissimilarity matrices~\citep{Pekalska05,DuinPekalska09}. Following the instructions in~\citet{Pekalska09}, the dissimilarity matrices are converted to similarities by applying the transformation characteristic to multi-dimensional scaling~\citep[e.g., see the negative double-centering transformation in][]{coxcox00}. In each simulation, we perform $10$-fold stratified cross-validation and measure the effectiveness of an approach with the average/median percentage of misclassified examples. For multi-class problems, we only evaluate the effectiveness of a single binary \emph{one-vs-all} classifier~\citep[just as in][Appendix C]{OglicG18icml}. Figure~\ref{fig:dissim-rank-vs-error} shows the reduction in the classification error as the approximation rank increases. The reported error is the median error over $10$-folds. Here, \acro{sf-lsm} represents the baseline in which similarities are used as features and a linear ridge regression model is trained in that instance space~\citep{Chen09,alabdulmohsin14}. The figure indicates that the baseline is quite competitive, but overall the proposed low-rank variants perform very well across different datasets (additional plots are provided in Appendix~\ref{app:plots}). Tables~\ref{tbl:baseline-results} provides the detailed results over all the datasets. In Appendix~\ref{app:plots} (Table~\ref{tbl:complete-app}), we also compare the effectiveness of our low-rank approaches with respect to the relevant state-of-the-art methods which make no approximations and represent hypotheses via the span of kernel functions centered at training instances. The empirical results indicate a competitive performance of our low-rank approaches with only $100$ landmarks across all the datasets and a variety of indefinite kernel functions.

\section{Discussion}
\label{sec:discussion}

The Nystr\"om method has recently been used for approximate eigendecomposition and low-rank approximation of indefinite kernel matrices~\citep{gisbrechtschleif15,SchleifT15,SchleifGT16}. To circumvent the fact that the original derivations of the approach are restricted to positive definite Mercer kernels~\citep{Smola00,Williams01},~\citet{gisbrechtschleif15} provide a derivation of the approach based on approximations of integral eigenfunctions arising in an eigendecomposition of an indefinite kernel. In particular, the authors of that work introduce an integral operator defined with an indefinite kernel and its empirical/sample-based approximation which asymptotically converges to the original (indefinite) integral operator. Based on this result,~\citet{gisbrechtschleif15} provide a derivation of the Nystr\"om method for indefinite kernels that treats the approximate equalities arising in the approximations of integral eigenfunctions as if they were exact. While such an assumption might hold for some datasets it fails to hold in the general case and this fact makes their extension of the Nystr\"om method to indefinite kernels mathematically incomplete. Our derivation of the approach does not rely on such an assumption and, thus, provides a stronger result. Moreover, our proof is much simpler than the one in~\citet{gisbrechtschleif15} and provides a geometrical intuition for the approximation. 

In addition to this,~\citet{gisbrechtschleif15,SchleifT15,SchleifGT16} proposed a method for finding an approximate low-rank eigendecomposition of an indefinite kernel matrix (for the sake of completeness, we review this approach in Appendix~\ref{app:addendum}). From the perspective of the exact number of floating point operations (\acro{flop}s), the approach by~\citet{gisbrechtschleif15,SchleifT15,SchleifGT16} requires $7$ matrix-to-matrix multiplications (each with the cost of $m^2n$ \acro{flop}s) and $2$ eigendecompositions (each with the cost of $m^3$ \acro{flop}s). Thus, in total their approach requires $7m^2n + 2m^3$ \acro{flop}s to find an approximate low-rank eigendecomposition of an indefinite kernel matrix. In contrast to this, the approach proposed in Section~\ref{subsec:nystroem} comes with a much better runtime complexity and requires at most $3m^2n + 3m^3$ \acro{flop}s. To see a practical runtime benefit of our approach, take a problem of approximating the kernel matrix defined with $n=10^6$ instances using $m=10^3$ landmarks. Our method for approximate low-rank eigendecomposition requires $3 \times 10^{12}$ less \acro{flop}s than the approach proposed by~\citet{gisbrechtschleif15,SchleifT15,SchleifGT16}.

Beside the considered low-rank approximations, it is possible to treat indefinite similarity functions as features and learn with linear models~\citep{alabdulmohsin14,Chen09} or squared kernel matrices~\citep{GraepelHBO98}. However,~\citet{Balcan2008} have showed that learning with a positive definite kernel corresponding to a feature space where the target concept is separable by a linear hypothesis yields a larger margin compared to learning with a linear model in a feature space constructed using that kernel function. As a result, if a kernel is used to construct a feature representation the sample complexity of a linear model in that space might be higher compared to learning with a kernelized variant of regularized risk minimization.

The effectiveness of a particular landmark selection strategy is a problem studied separately from the derivation of the Nystr\"om method and we, therefore, do not focus on that problem in this work. However, clustering and leverage score sampling have been proposed and validated in earlier publications and are state-of-the-art for low-rank approximation of positive definite kernels~\cite{Kumar12,Alaoui15,Gittens16,OglicG17icml}. As the flip-spectrum matrix shares the eigenspace with the indefinite kernel matrix, the convergence results on the effectiveness of landmark selection strategies for Nystr\"om low-rank approximation of positive definite kernels apply to indefinite kernels~\citep[e.g., see Section~\ref{subsec:landmarks} or][]{eckart1936,mirsky1960}. In particular, bounds for the leverage score sampling strategy applied to the flip-spectrum matrix carry over to our derivation of the Nystr\"om method for indefinite kernels.

We conclude with a reference to~\citet{Schleif18} and~\citet{Loosli16}, where an issue concerning the sparsity of a solution returned by the \krein support vector machine has been raised. We hypothesize that our approach can overcome this limitation by either controlling the approximation rank or penalizing the low-rank objective with the $\ell_1$-norm of the linear model. We leave the theoretical study and evaluation of such an approach for future work.

\nocite{OglicThesis,Schleif17ICVM}
{
\small\selectfont
\textbf{Acknowledgments:} We are grateful for access to the University of Nottingham High Performance Computing Facility. Dino Oglic was supported in part by EPSRC grant EP/R012067/1.
}

{
\small\selectfont
\bibliography{./krein-nystroem}
\bibliographystyle{apalike}
}

\appendix

\section{Proofs}
\label{app:proofs}

\propKreinRegression*

\begin{proof}
The optimal hypothesis over training data satisfies
\begin{align*}
\begin{aligned}
& K\alpha^{*} = && U D U^{\top} U\brackets{DS + n\lambda\mathbb{I}}^{-1} U^{\top} USU^{\top}y = & \\
& && U DS \brackets{DS + n\lambda\mathbb{I}}^{-1} U^{\top} y =  & \\
& && H \brackets{H + n\lambda\mathbb{I}}^{-1} y \ . &
\end{aligned}
\end{align*}

Thus, if we only regularize with $\norm{f}_{\mathcal{H}_{\mathcal{K}}}$ then the \krein kernel ridge regression problem is equivalent to that with the flip-spectrum transformation combined with the corresponding out-of-sample extension. More formally, if $\alpha_{H}^{*}=\brackets{H + n\lambda\mathbb{I}}^{-1} y$ denotes the optimal solution of the kernel ridge regression with the flip-spectrum matrix $H$ in place of the indefinite kernel matrix $K$ then the predictions at out-of-sample test instances are given by 
\begin{align*}
f\brackets{x}=k_{x}^{\top} P \alpha_{H}^{*} \ .
\end{align*}
\end{proof}

\propKreinSVM*

\begin{proof}
	The optimal hypothesis over training instances is
	\begin{align*}
	K\alpha^{*} = KP\alpha_H^{*} = UDU^{\top} USU^{\top} \alpha_H^{*} = H \alpha_H^{*} \ ,
	\end{align*}
	where $\alpha_{H}^{*}$ is the optimal solution for the support vector machine problem with the flip-spectrum matrix $H$ in place of the indefinite \krein kernel matrix $K$. Thus, this variant of \krein support vector machine is equivalent to learning with the flip-spectrum transformation of an indefinite kernel matrix. An out-of-sample extension for a test instance $x$ is
	\begin{align*}
	f\brackets{x}=k_{x}^{\top} \alpha^{*} = k_x^{\top} P \alpha_{H}^{*} \ .
	\end{align*}
\end{proof}

\section{Discussion Addendum}
\label{app:addendum}

We provide here a brief review of the approach by~\citet{gisbrechtschleif15,SchleifT15,SchleifGT16} for approximate eigendecomposition of an indefinite matrix~\footnote{\tiny \url{https://www.techfak.uni-bielefeld.de/~fschleif/eigenvalue\_corrections\_demos.tgz}, accessed in May $2018$}. The approach is motivated by the observation that an indefinite symmetric matrix and its square have identical eigenvectors. For this reason, the authors first form the squared low-rank \krein kernel matrix 
\begin{align*}
\tilde{K}^2 = K_{X \times Z} K_{Z \times Z}^{-1} K_{Z \times X} K_{X \times Z} K_{Z \times Z}^{-1} K_{Z \times X} \ .
\end{align*}
The matrix $A=K_{Z \times Z}^{-1} K_{Z \times X} K_{X \times Z} K_{Z \times Z}^{-1}$ is positive definite because it can be written as $LL^{\top}$ (e.g., taking $L=K_{Z \times Z}^{-1} K_{Z \times X}$). Thus, all the eigenvalues in an eigendecomposition of $A=V \Gamma V^{\top}$ are non-negative and we can set $A=LL^{\top}$ with $L=V \Gamma^{\frac{1}{2}}$. From here it then follows that the matrix $\tilde{K}^2$ can be factored as $\tilde{K}^2 = BB^{\top}$ with $B=K_{X \times Z} V \Gamma^{\frac{1}{2}}$. Following this,~\citet{gisbrechtschleif15,SchleifT15,SchleifGT16} mimic the standard procedure for the derivation of approximate eigenvectors and eigenvalues characteristic to the Nystr\"om method for positive definite kernels~\citep{Williams01,Fowlkes04,Drineas05,Drineas06}. In particular, they first decompose the positive definite matrix $B^{\top}B=Q \Delta Q^{\top}$ and then compute the approximate eigenvectors of $\tilde{K}^2$ as $\tilde{U}=B Q \Delta^{-\frac{1}{2}}$. Now, to obtain an approximate eigendecomposition of the \krein kernel matrix the authors use these eigenvectors in combination with the posited form of the low-rank approximation $\tilde{K}=K_{X \times Z} K_{Z \times Z}^{-1} K_{Z \times X}$ and compute the approximate eigenvalues as $ \tilde{D} = \tilde{U}^{\top} \tilde{K} \tilde{U}$. As the diagonal matrix $\Delta$ contains the eigenvalues of $\tilde{K}^2$ this step retrieves the signed eigenvalues of $\tilde{K}$. The \emph{one-shot} Nystr\"om approximation of the kernel matrix is then given as~\citep{gisbrechtschleif15,SchleifT15,SchleifGT16}
\begin{align*}
\begin{aligned}
& K_{X \mid Z}^{\textsc{sgt}^*} = \tilde{U} \tilde{D} \tilde{U}^{\top} = & \\
& K_{X \times Z} V \Gamma^{\frac{1}{2}} Q \Delta^{-\frac{1}{2}} \ \tilde{D} \ \Delta^{-\frac{1}{2}} Q^{\top} \Gamma^{\frac{1}{2}} V^{\top} K_{Z \times X} \ . &
\end{aligned}
\end{align*}

\clearpage

\onecolumn

\section{Additional Experiments}
\label{app:plots}

\begin{figure*}[!htb]
	\centering
	\input{./simple-sigmoid-rank-vs-error.tikz}
	\caption{The figure shows the reduction in the approximation error for an indefinite kernel matrix obtained using the \acro{Sigmoid} kernel~\citep[][Appendix D]{OglicG18icml}, which comes as a result of the increase in the approximation rank. In the figure legend, we use $\brackets{k \mid l}$ to express the fact that a rank $k$ approximation of the kernel matrix is computed using a set of $l$ landmarks.}
	\label{fig:simple-sigmoid-rank-vs-error}
\end{figure*}

\vfill

\begin{figure*}[!htb]
	\centering
	\input{./sigmoid-rank-vs-error.tikz}
	\caption{The figure shows the reduction in the approximation error for an indefinite kernel matrix obtained using the \acro{RL-Sigmoid} kernel~\citep[][Appendix D]{OglicG18icml}, which comes as a result of the increase in the approximation rank. In the figure legend, we use $\brackets{k \mid l}$ to express the fact that a rank $k$ approximation of the kernel matrix is computed using a set of $l$ landmarks.}
	\label{fig:sigmoid-rank-vs-error}
\end{figure*}

\vfill

\begin{figure*}[!htb]
	\centering
	\input{./epanechnikov-rank-vs-error.tikz}
	\caption{The figure shows the reduction in the approximation error for an indefinite kernel matrix obtained using the \acro{Epanechnikov} kernel~\citep[][Appendix D]{OglicG18icml}, which comes as a result of the increase in the approximation rank. In the figure legend, we use $\brackets{k \mid l}$ to express the fact that a rank $k$ approximation of the kernel matrix is computed using a set of $l$ landmarks.}
	\label{fig:epanechnikov-rank-vs-error}
\end{figure*}

\clearpage

\begin{figure*}[!htb]
	\centering
	\begin{tikzpicture}[font=\tiny]
\hypersetup{hidelinks}
\begin{groupplot}[group style={group size=5 by 1, horizontal sep=2.3em}]

\nextgroupplot[
title={(a) zongker, $\iota=59\%$},
height=2.2cm,
width=2.5cm,
scale only axis,
axis line style={-},
tick label style={font=\tiny},
label style={font=\tiny},
legend style={font=\tiny, at={(3.0, 1.4)}, anchor=south, draw=none, legend cell align=left, legend columns=4, column sep=2em},
legend image post style={scale=1.2},
xmin=1.5, xmax=4.7,
ymin=14, ymax=46,
axis on top,
xtick={1.5,  2.3,  3.1,  3.9,  4.7},
xticklabels={5, 10,  25,  50,  100},
ytick={ 14.,  22.,  30.,  38.,  46},
yticklabels={ 14,  22,  30,  38,  46},
tick align=outside,
xlabel={rank (log scale)},
ylabel={classification error (\%)}
]

\addplot [c-km, solid, line width=0.5pt, mark=*, mark size=1pt, mark repeat={1}, mark phase=1]
coordinates {
	( 1.60943791243 , 44.5 )
	( 2.30258509299 , 34.25 )
	( 3.21887582487 , 29.75 )
	( 3.91202300543 , 19.5 )
	( 4.60517018599 , 16.0 )
};
\addlegendentry{\acro{sf-lsm}};

\addplot [c-lss-1, solid, line width=0.5pt, mark=square*, mark size=1pt, mark repeat={1}, mark phase=1]
coordinates {
	( 1.60943791243 , 36.25 )
	( 2.30258509299 , 35.5 )
	( 3.21887582487 , 30.0 )
	( 3.91202300543 , 21.5 )
	( 4.60517018599 , 18.25 )
};
\addlegendentry{\krein \acro{lsm}};

\addplot [c-dpp-1, solid, line width=0.5pt, mark=diamond*, mark size=1.5pt, mark repeat={1}, mark phase=1]
coordinates {
	( 1.60943791243 , 36.25 )
	( 2.30258509299 , 35.75 )
	( 3.21887582487 , 30.75 )
	( 3.91202300543 , 21.5 )
	( 4.60517018599 , 17.75 )
};
\addlegendentry{\krein \acro{vc-lsm}};

\addplot [lloyd-km, solid, line width=0.5pt, mark=triangle*, mark size=1.5pt, mark repeat={1}, mark phase=1]
coordinates {
( 1.6094379124341003 , 39.75 )
( 2.302585092994046 , 36.0 )
( 3.2188758248682006 , 27.500000000000004 )
( 3.912023005428146 , 20.75 )
( 4.605170185988092 , 16.0 )
};
\addlegendentry{\krein \acro{sh-svm}};

\nextgroupplot[
title={(b) chicken25-45, $\iota=32\%$},
height=2.2cm,
width=2.5cm,
scale only axis,
axis line style={-},
tick label style={font=\tiny},
label style={font=\tiny},
legend style={font=\tiny},
xmin=1.5, xmax=4.7,
ymin=15, ymax=39,
axis on top,
xtick={1.5,  2.3,  3.1,  3.9,  4.7},
xticklabels={5, 10,  25,  50,  100},
ytick={  15.,  21.,  27.,  33.,  39},
yticklabels={    15,  21,  27,  33,  39},
tick align=outside,
xlabel={rank (log scale	)}
]

\addplot [c-km, solid, line width=0.5pt, mark=*, mark size=1.pt, mark repeat={1}, mark phase=1]
coordinates {
( 1.60943791243 , 27.2727272727 )
( 2.30258509299 , 31.4393939394 )
( 3.21887582487 , 26.6666666667 )
( 3.91202300543 , 21.3636363636 )
( 4.60517018599 , 16.8686868687 )
};

\addplot [c-lss-1, solid, line width=0.5pt, mark=square*, mark size=1pt, mark repeat={1}, mark phase=1]
coordinates {
( 1.60943791243 , 34.8232323232 )
( 2.30258509299 , 32.601010101 )
( 3.21887582487 , 23.6111111111 )
( 3.91202300543 , 20.4545454545 )
( 4.60517018599 , 20.0 )
};

\addplot [c-dpp-1, solid, line width=0.5pt, mark=diamond*, mark size=1.5pt, mark repeat={1}, mark phase=1]
coordinates {
( 1.60943791243 , 34.8232323232 )
( 2.30258509299 , 31.4646464646 )
( 3.21887582487 , 22.4747474747 )
( 3.91202300543 , 23.3333333333 )
( 4.60517018599 , 17.9797979798 )
};

\addplot [lloyd-km, solid, line width=0.5pt, mark=triangle*, mark size=1.5pt, mark repeat={1}, mark phase=1]
coordinates {
( 1.6094379124341003 , 39.0 )
( 2.302585092994046 , 27.2727272727 )
( 3.2188758248682006 , 23.5858585859 )
( 3.912023005428146 , 25.5555555556 )
( 4.605170185988092 , 22.5555555556 )
};

\nextgroupplot[
title={(c) polydish57, $\iota=42\%$},
height=2.2cm,
width=2.5cm,
scale only axis,
axis line style={-},
tick label style={font=\tiny},
label style={font=\tiny},
legend style={font=\tiny},
xmin=1.5, xmax=4.7,
ymin=2, ymax=42,
axis on top,
xtick={1.5,  2.3,  3.1,  3.9,  4.7},
xticklabels={5, 10,  25,  50,  100},
ytick={  2.,  12.,  22.,  32.,  42.},
yticklabels={  2,  12,  22,  32,  42},
tick align=outside,
xlabel={rank (log scale)}
]
     
\addplot [c-km, solid, line width=0.5pt, mark=*, mark size=1pt, mark repeat={1}, mark phase=1]
coordinates {
( 1.60943791243 , 40.25 )
( 2.30258509299 , 28.75 )
( 3.21887582487 , 12.25 )
( 3.91202300543 , 7.375 )
( 4.60517018599 , 2.875 )
};

\addplot [c-lss-1, solid, line width=0.5pt, mark=square*, mark size=1pt, mark repeat={1}, mark phase=1]
coordinates {
( 1.60943791243 , 34.125 )
( 2.30258509299 , 26.75 )
( 3.21887582487 , 18.5 )
( 3.91202300543 , 5.75 )
( 4.60517018599 , 3.5 )
};

\addplot [c-dpp-1, solid, line width=0.5pt, mark=diamond*, mark size=1pt, mark repeat={1}, mark phase=1]
coordinates {
( 1.60943791243 , 34.25 )
( 2.30258509299 , 26.75 )
( 3.21887582487 , 18.5 )
( 3.91202300543 , 5.75 )
( 4.60517018599 , 3.5 )
};

\addplot [lloyd-km, solid, line width=0.5pt, mark=triangle*, mark size=1.5pt, mark repeat={1}, mark phase=1]
coordinates {
( 1.6094379124341003 , 30.125 )
( 2.302585092994046 , 27.875 )
( 3.2188758248682006 , 10.5 )
( 3.912023005428146 , 5.625 )
( 4.605170185988092 , 2.75 )
};

\nextgroupplot[
title={(d) polydism57, $\iota=36\%$},
height=2.2cm,
width=2.5cm,
scale only axis,
axis line style={-},
tick label style={font=\tiny},
label style={font=\tiny},
legend style={font=\tiny},
xmin=1.5, xmax=4.7,
ymin=0, ymax=36,
axis on top,
xtick={1.5,  2.3,  3.1,  3.9,  4.7},
xticklabels={5, 10,  25,  50,  100},
ytick={  0.,   9.,  18.,  27.,  36.},
yticklabels={    0,   9,  18,  27,  36},
tick align=outside,
xlabel={rank (log scale	)}
]

\addplot [c-km, solid, line width=0.5pt, mark=*, mark size=1.pt, mark repeat={1}, mark phase=1]
coordinates {
	( 1.60943791243 , 25.625 )
	( 2.30258509299 , 7.25 )
	( 3.21887582487 , 1.375 )
	( 3.91202300543 , 0.625 )
	( 4.60517018599 , 0.625 )
};

\addplot [c-lss-1, solid, line width=0.5pt, mark=square*, mark size=1pt, mark repeat={1}, mark phase=1]
coordinates {
	( 1.60943791243 , 33.375 )
	( 2.30258509299 , 8.625 )
	( 3.21887582487 , 2.75 )
	( 3.91202300543 , 0.75 )
	( 4.60517018599 , 0.375 )
};

\addplot [c-dpp-1, solid, line width=0.5pt, mark=diamond*, mark size=1.5pt, mark repeat={1}, mark phase=1]
coordinates {
	( 1.60943791243 , 33.375 )
	( 2.30258509299 , 8.5 )
	( 3.21887582487 , 2.75 )
	( 3.91202300543 , 0.75 )
	( 4.60517018599 , 0.375 )
};

\addplot [lloyd-km, solid, line width=0.5pt, mark=triangle*, mark size=1.5pt, mark repeat={1}, mark phase=1]
coordinates {
( 1.6094379124341003 , 20.625 )
( 2.302585092994046 , 19.625 )
( 3.2188758248682006 , 1.7500000000000002 )
( 3.912023005428146 , 0.75 )
( 4.605170185988092 , 0.25 )
};

\nextgroupplot[
title={(e) woodyplants50, $\iota=23\%$},
height=2.2cm,
width=2.5cm,
scale only axis,
axis line style={-},
tick label style={font=\tiny},
label style={font=\tiny},
legend style={font=\tiny},
xmin=1.5, xmax=4.7,
ymin=25, ymax=45,
axis on top,
xtick={1.5,  2.3,  3.1,  3.9,  4.7},
xticklabels={5, 10,  25,  50,  100},
ytick={  25.,  30.,  35.,  40.,  45.},
yticklabels={  25,  30,  35,  40,  45},
tick align=outside,
xlabel={rank (log scale)}
]

\addplot [c-km, solid, line width=0.5pt, mark=*, mark size=1pt, mark repeat={1}, mark phase=1]
coordinates {
( 1.60943791243 , 44.375 )
( 2.30258509299 , 32.0512820513 )
( 3.21887582487 , 36.4715189873 )
( 3.91202300543 , 30.8227848101 )
( 4.60517018599 , 26.5705128205 )
};

\addplot [c-lss-1, solid, line width=0.5pt, mark=square*, mark size=1pt, mark repeat={1}, mark phase=1]
coordinates {
( 1.60943791243 , 43.0379746835 )
( 2.30258509299 , 36.7147435897 )
( 3.21887582487 , 32.9113924051 )
( 3.91202300543 , 28.75 )
( 4.60517018599 , 30.1819620253 )
};
      
\addplot [c-dpp-1, solid, line width=0.5pt, mark=diamond*, mark size=1.5pt, mark repeat={1}, mark phase=1]
coordinates {
( 1.60943791243 , 43.3939873418 )
( 2.30258509299 , 35.0373255437 )
( 3.21887582487 , 32.2784810127 )
( 3.91202300543 , 27.2115384615 )
( 4.60517018599 , 30.8148734177 )
};

\addplot [lloyd-km, solid, line width=0.5pt, mark=triangle*, mark size=1.5pt, mark repeat={1}, mark phase=1]
coordinates {
( 1.6094379124341003 , 41.4070107108 )
( 2.302585092994046 , 37.3397435897 )
( 3.2188758248682006 , 37.9746835443 )
( 3.912023005428146 , 38.6075949367 )
( 4.605170185988092 , 37.2628205128 )
};

\end{groupplot}
\end{tikzpicture}
	\caption{The figure shows the reduction in the classification error as the approximation rank of a Nystr\"om low-rank approximation increases. The reported error is the median classification error obtained using $10$-fold stratified cross-validation.}
	\label{fig:app-dissim-rank-vs-error}
\end{figure*}
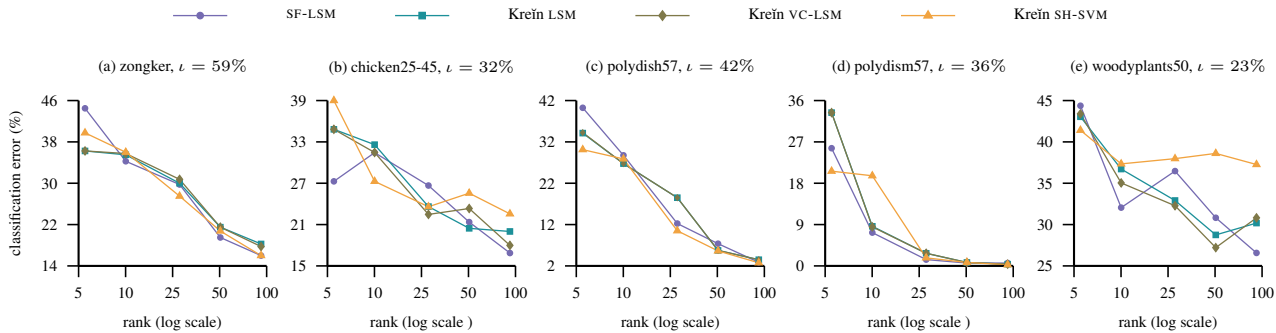

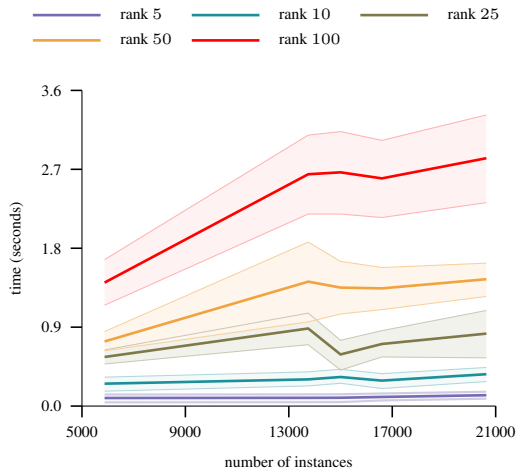
\begin{figure*}[!htb]
	\centering
	\begin{tikzpicture}[font=\tiny]
\hypersetup{hidelinks}
\begin{groupplot}[group style={group size=1 by 1, horizontal sep=2.4em}]

\nextgroupplot[
title={},
height=4.2cm,
width=5.5cm,
scale only axis,
axis line style={-},
tick label style={font=\tiny},
label style={font=\tiny},
legend style={font=\tiny, at={(0.45, 1.1)}, anchor=south, draw=none, legend cell align=left, legend columns=3, column sep=.5em},
legend image post style={scale=1.5},
xmin=5000, xmax=21000,
ymin=0, ymax=3.6,
axis on top,
xtick={5000.,  9000., 13000., 17000., 21000.},
xticklabels={5000,  9000, 13000, 17000, 21000},
ytick={0. , 0.9, 1.8, 2.7, 3.6},
yticklabels={0.0 , 0.9, 1.8, 2.7, 3.6},
tick align=outside,
xlabel={number of instances},
ylabel={time (seconds)},
scaled x ticks=false
]


\addplot[c-km, line width=1]
coordinates {
(5875, 0.09194194602966309)
(13750, 0.09394194602966309)
(15000, 0.09494194602966309 )
(16599, 0.10394194602966309)
(20640, 0.12394194602966309)
};
\addlegendentry{rank $5$};
     
\addplot[c-lss-1, line width=1]
coordinates {
(5875, 0.2554636478424072)
(13750, 0.30443606376647947)
(15000, 0.33133299350738527 )
(16599, 0.29033443927764893)
(20640, 0.36328110694885254)
};
\addlegendentry{rank $10$};

\addplot[c-dpp-1, line width=1]
coordinates {
(5875,0.5598509550094605)
(13750, 0.8847391366958618)
(15000, 0.5890582799911499)
(16599, 0.7080075979232788)
(20640, 0.8262929201126098)
};
\addlegendentry{rank $25$};

\addplot[lloyd-km, line width=1]
coordinates {
	(5875, 0.7366926908493042)
	(13750,  1.4188207149505616)
	(15000, 1.3512642145156861 )
	(16599, 1.342190647125244)
	(20640, 1.4466442108154297)
};
\addlegendentry{rank $50$};

\addplot[red, line width=1]
coordinates {
	(5875, 1.4081483840942384)
	(13750, 2.6438645362854003)
	(15000, 2.6656078338623046 )
	(16599, 2.5966380834579468)
	(20640, 2.82677686214447)
};
\addlegendentry{rank $100$};

\addplot[name path=c-km-lower-bound, draw=c-km!35, mark=none]
coordinates {
(5875, 0.04194194602966309)
(13750, 0.04394194602966309)
(15000, 0.04494194602966309 )
(16599, 0.06394194602966309)
(20640, 0.08394194602966309)
};

\addplot[name path=c-km-upper-bound, draw=c-km!35, mark=none]
coordinates {
(5875, 0.13194194602966309)
(13750, 0.13394194602966309)
(15000, 0.13494194602966309 )
(16599, 0.14394194602966309)
(20640, 0.16394194602966309)
};

\addplot [color=c-km, fill=c-km!5] fill between[of=c-km-lower-bound and c-km-upper-bound];

\addplot[name path=c-lss-1-lower-bound, draw=c-lss-1!35, mark=none]
coordinates {
	(5875, 0.17)
	(13750, 0.23)
	(15000, 0.26 )
	(16599, 0.20)
	(20640, 0.28)
};

\addplot[name path=c-lss-1-upper-bound, draw=c-lss-1!35, mark=none]
coordinates {
	(5875, 0.33)
	(13750, 0.39)
	(15000, 0.42 )
	(16599, 0.37)
	(20640, 0.44)
};

\addplot [color=c-lss-1, fill=c-lss-1!5] fill between[of=c-lss-1-lower-bound and c-lss-1-upper-bound];

\addplot[name path=c-dpp-1-lower-bound, draw=c-dpp-1!40, mark=none]
coordinates {
	(5875, 0.48)
	(13750, 0.70)
	(15000, 0.41 )
	(16599, 0.56)
	(20640, 0.55)
};

\addplot[name path=c-dpp-1-upper-bound, draw=c-dpp-1!40, mark=none]
coordinates {
	(5875, 0.64)
	(13750, 1.06)
	(15000, 0.75 )
	(16599, 0.86)
	(20640, 1.09)
};

\addplot [color=c-dpp-1, fill=c-dpp-1!10] fill between[of=c-dpp-1-lower-bound and c-dpp-1-upper-bound];

\addplot[name path=lloyd-km-lower-bound, draw=lloyd-km!50, mark=none]
coordinates {
(5875, 0.63)
(13750, 0.96)
(15000, 1.05 )
(16599, 1.10)
(20640, 1.25)
};

\addplot[name path=lloyd-km-upper-bound, draw=lloyd-km!50, mark=none]
coordinates {
(5875, 0.85)
(13750, 1.87)
(15000, 1.65 )
(16599, 1.58)
(20640, 1.63)
};

\addplot [color=lloyd-km, fill=lloyd-km!10] fill between[of=lloyd-km-lower-bound and lloyd-km-upper-bound];

\addplot[name path=red-lower-bound, draw=red!30, mark=none]
coordinates {
(5875, 1.15)
(13750, 2.19)
(15000, 2.19 )
(16599, 2.15)
(20640, 2.32)
};

\addplot[name path=red-upper-bound, draw=red!30, mark=none]
coordinates {
(5875, 1.67)
(13750, 3.09)
(15000, 3.13 )
(16599, 3.03)
(20640, 3.32)
};

\addplot [color=red, fill=red!5] fill between[of=red-lower-bound and red-upper-bound];

\end{groupplot}
\end{tikzpicture}
	\caption{The figure depicts the computational cost for Nystr\"om approximations of different ranks as a function of the number of instances in a dataset. The reported time is the average time required to compute the Nystr\"om approximation of a given rank (averaged over $10$ repetitions of the experiment). The confidence interval for a \emph{cost-curve} is computed by subtracting/adding the corresponding standard deviations from the average computational costs. The landmarks are selected using the approximate kernel $K$-means$++$ sampling strategy proposed in Section~\ref{subsec:landmarks}. For all the considered approximation ranks, the \emph{cost-curves} indicate that the approach scales (approximately) linearly with respect to the dataset size (the slope of a \emph{cost-curve} depends on the approximation rank).}
	\label{fig:app-size-vs-time}
\end{figure*}

\begin{table*}[!htb]
	\centering
	\fontsize{7.8}{9.5}\selectfont
\setlength{\tabcolsep}{3pt}
\begin{tabular}{l|rr|rr|rr|rr|rr|rr|rr}

\hline

\hline

\multirow{2}{*}{\sc Dataset} & 
\multicolumn{8}{c|}{\sc Rank $100$ Approximation (Primal Optimization)} &
\multicolumn{6}{c}{\sc Whole \krein Space (Dual Optimization)} \\\cline{2-15}

 & 
\multicolumn{2}{c|}{\textsc{\krein vc-lsm}} &
\multicolumn{2}{c|}{\textsc{\krein lsm}} &
\multicolumn{2}{c|}{\textsc{\krein sh-svm}} &
\multicolumn{2}{c|}{\textsc{sf-lsm}} &
\multicolumn{2}{c|}{\textsc{k-svm}}  &
\multicolumn{2}{c}{\textsc{sf-lsm}} &
\multicolumn{2}{c}{\textsc{vc-lsm}} \\

\hline

\hline

coilyork & 
$32.22 $ & $(\pm 7.89)$ &
$31.21 $ & $(\pm 5.28)$ & 
$38.20 $ & $(\pm 7.20)$ &
$35.33 $ & $(\pm 10.09)$ & 
$32.91 $ & $(\pm 8.06)$ & 
$26.03 $ & $(\pm 5.60)$ &
$22.56 $ & $(\pm 7.66)$ \\

balls 3D & 
$1.00 $ & $(\pm 2.00)$ &
$0.50 $ & $(\pm 1.50)$ &
$0.00 $ & $(\pm 0.00)$ &
$0.50 $ & $(\pm 1.50)$ & 
$0.00 $ & $(\pm 0.00)$ & 
$0.00 $ & $(\pm 0.00)$ &
$0.00$ & $(\pm 0.00)$\\

prodom & 
$0.92 $ & $(\pm 0.46)$ &
$0.92 $ & $(\pm 0.46)$ & 
$0.54 $ & $(\pm 0.47)$ &
$1.57 $ & $(\pm 0.58)$ & 
$0.00 $ & $(\pm 0.00)$ & 
$0.04 $ & $(\pm 0.11)$ &
$0.00$ & $(\pm 0.00)$\\

chicken10 & 
$16.35 $ & $(\pm 4.31)$ &
$15.69 $ & $(\pm 4.97)$ & 
$16.82 $ & $(\pm 6.57)$ &
$14.37 $ & $(\pm 4.02)$ & 
$30.95 $ & $(\pm 7.81)$ & 
$11.91 $ & $(\pm 3.56)$ &
$5.62$ & $(\pm 2.55)$ \\

protein & 
$4.19 $ & $(\pm 2.47)$ &
$3.72 $ & $(\pm 2.76)$ & 
$5.23 $ & $(\pm 2.89)$ & 
$5.15 $ & $(\pm 3.91)$ & 
$5.17 $ & $(\pm 3.34)$ &
$2.83 $ & $(\pm 3.15)$ &
$0.00$ & $(\pm 0.00)$\\

zongker & 
$17.70 $ & $(\pm 2.06)$ &
$17.75 $ & $(\pm 2.23)$ & 
$15.30 $ & $(\pm 3.39)$ &
$17.05 $ & $(\pm 2.36)$ & 
$16.00 $ & $(\pm 1.41)$ &
$5.60 $ & $(\pm 1.20)$ &
$0.95$ & $(\pm 1.68)$\\

chicken25 & 
$19.29 $ & $(\pm 4.64)$ &
$20.41 $ & $(\pm 4.09)$ & 
$25.77 $ & $(\pm 4.68)$ &
$18.17 $ & $(\pm 6.67)$ & 
$17.72 $ & $(\pm 6.57)$ &
$16.38 $ & $(\pm 5.14)$ &
$4.73$ & $(\pm 3.29)$\\

pdish57 & 
$3.40 $ & $(\pm 0.39)$ &
$3.40 $ & $(\pm 0.42)$ & 
$2.73 $ & $(\pm 0.62)$ &
$3.03 $ & $(\pm 0.67)$ & 
$0.42 $ & $(\pm 0.25)$ &
$0.20 $ & $(\pm 0.19)$ &
$0.35$ & $(\pm 0.37)$ \\

pdism57 & 
$0.38 $ & $(\pm 0.26)$ &
$0.38 $ & $(\pm 0.26)$ & 
$0.30 $ & $(\pm 0.29)$ &
$0.63 $ & $(\pm 0.42)$ & 
$0.13 $ & $(\pm 0.23)$ &
$0.15 $ & $(\pm 0.17)$ &
$0.11$ & $(\pm 0.18)$\\

woody50 & 
$30.84 $ & $(\pm 5.25)$ &
$30.47 $ & $(\pm 5.54)$ & 
$38.42 $ & $(\pm 7.13)$ &
$26.41 $ & $(\pm 4.42)$ & 
$37.04 $ & $(\pm 5.07)$ &
$22.89 $ & $(\pm 4.07)$ &
$2.53$ & $(\pm 2.66)$\\

\hline

\hline


\multicolumn{15}{l}{\fontsize{7}{10}\selectfont \textsc{k-svm} denotes the \krein support vector machine~\citep{Loosli16}.} \\

\multicolumn{15}{l}{\fontsize{7}{10}\selectfont \textsc{sf-lsm} denotes a variant of the least squares method with similarities as features~\citep{GraepelHBO98,Chen09,alabdulmohsin14}.} \\

\multicolumn{15}{l}{\fontsize{7}{10}\selectfont \textsc{vc-lsm} denotes the variance constrained least squares method~\citep{OglicG18icml}.}

\end{tabular}
	\caption{The table reports the results of our experiments on benchmark datasets for learning with indefinite kernels~\citep{Pekalska05}. The effectiveness of an approach is measured using the average classification error obtained via $10$-fold stratified cross-validation. In contrast to the experiments over the whole \krein space~\citep{OglicG18icml}, the hyper-parameter optimization for low-rank approaches did not involve any random restarts (which could further improve the reported results for considered \krein methods).}
	\label{tbl:complete-app}
\end{table*}

\end{document}